\documentclass{article}

\usepackage[accepted]{icml2026}

\usepackage{microtype}
\usepackage{graphicx}
\usepackage{subcaption}
\usepackage{booktabs} 
\usepackage{hyperref}

\usepackage{amsmath}
\usepackage{amssymb}
\usepackage{mathtools}
\usepackage{amsthm}

\usepackage[capitalize,noabbrev]{cleveref}

\usepackage{dsfont}
\usepackage[mathscr]{euscript}
\usepackage{xcolor}
\usepackage{colortbl}
\usepackage{thmtools}
\usepackage{thm-restate}

\usepackage{multirow}
\usepackage{wrapfig}

\usepackage{pifont}

\let\vec\undefined
\usepackage{MnSymbol}
\DeclareMathAlphabet\mathbb{U}{msb}{m}{n}
\usepackage{xpatch}

\def\Rset{\mathbb{R}}

\DeclareMathOperator*{\E}{\mathbb E}
\DeclareMathOperator*{\argmax}{argmax}

\DeclarePairedDelimiter{\abs}{\lvert}{\rvert} 
\DeclarePairedDelimiter{\bracket}{[}{]}
\DeclarePairedDelimiter{\curl}{\{}{\}}
\DeclarePairedDelimiter{\paren}{(}{)}
\DeclarePairedDelimiter{\norm}{\|}{\|}

\newcommand{\sD}{{\mathscr D}}
\newcommand{\sE}{{\mathscr E}}
\newcommand{\sF}{{\mathscr F}}

\newcommand{\sH}{{\mathscr H}}

\newcommand{\sM}{{\mathscr M}}
\newcommand{\sP}{{\mathscr P}}

\newcommand{\sX}{{\mathscr X}}
\newcommand{\sY}{{\mathscr Y}}

\newcommand{\sfp}{{\mathsf p}}

\newcommand{\sfF}{{\mathsf F}}
\newcommand{\sfL}{{\mathsf L}}

\newcommand{\sfX}{{\mathsf X}}

\newcommand{\Rad}{\mathfrak R}

\newcommand{\brho}{{\boldsymbol \rho}}
\newcommand{\bepsilon}{{\boldsymbol \e}}

\newcommand{\ff}{{\sf f}}

\newcommand{\expert}{{g}}
\newcommand{\expertexpert}{{\sf g}}
\newcommand{\ldef}{{\sfL_{\rm{def}}}}

\newcommand{\tdef}{{\sfL_{\rm{def}}}}

\newcommand{\num}{{p}}

\newcommand{\TDEF}{\textsc{tdef}}
\newcommand{\MILD}{\textsc{mild}}

\newcommand{\h}{\widehat}
\newcommand{\ov}{\overline}

\newcommand{\wt}{\widetilde}
\newcommand{\e}{\epsilon}
\newcommand{\NA}{---}
\newcommand{\ignore}[1]{}

\theoremstyle{plain}
\newtheorem{theorem}{Theorem}[section]

\newtheorem{corollary}[theorem]{Corollary}
\theoremstyle{definition}

\theoremstyle{remark}

\usepackage[disable,textsize=tiny]{todonotes}

\hypersetup{
  breaklinks   = true, 
  colorlinks   = true, 
  urlcolor     = blue, 
  linkcolor    = blue, 
  citecolor   = blue 
}

\usepackage[toc, page, header]{appendix}
\icmltitlerunning{Optimized Deferral for Imbalanced Settings}

\begin{document}

\twocolumn[
  \icmltitle{Optimized Deferral for Imbalanced Settings}

\begin{icmlauthorlist}
\icmlauthor{Corinna Cortes}{google}
\icmlauthor{Anqi Mao}{courant}
\icmlauthor{Mehryar Mohri}{google,courant}
\icmlauthor{Yutao Zhong}{google}
\end{icmlauthorlist}

\icmlaffiliation{google}{Google Research, New York, NY;}
\icmlaffiliation{courant}{Courant Institute of Mathematical Sciences, New York, NY}

\icmlcorrespondingauthor{Corinna Cortes}{corinna@google.com}
\icmlcorrespondingauthor{Anqi Mao}{aqmao@cims.nyu.edu}
\icmlcorrespondingauthor{Mehryar Mohri}{mohri@google.com}
\icmlcorrespondingauthor{Yutao Zhong}{yutaozhong@google.com}

\icmlkeywords{learning to defer, cost-sensitive learning, margin bound, surrogate loss, consistency, learning theory}

\vskip 0.3in
]

\printAffiliationsAndNotice{}

\begin{abstract}

Learning algorithms can be significantly improved by routing complex
or uncertain inputs to specialized experts, balancing accuracy with
computational cost.  This approach, known as \emph{learning to defer},
is essential in domains like natural language generation, medical
diagnosis, and computer vision, where an effective deferral can reduce
errors at low extra resource consumption.  However, the two-stage learning to
defer setting, which leverages existing predictors such as a
collection of LLMs or other classifiers, often faces challenges due to an expert imbalance
problem.\ignore{, where some experts apply to only pockets of the input space.}
This imbalance can lead to suboptimal performance, with deferral
algorithms favoring the majority expert.
We present a comprehensive study of two-stage learning to defer in
expert imbalance settings. We cast the deferral loss optimization as a novel
cost-sensitive learning problem over the input-expert domain. \ignore{This motivates
our analysis of cost-sensitive multi-class classification with
imbalanced data.} We derive new margin-based loss functions and
guarantees tailored to this setting, and develop novel algorithms for
cost-sensitive learning. Leveraging these results, we design
principled deferral algorithms, \MILD\ (\emph{Margin-based Imbalanced
Learning to Defer}), specifically suited for expert imbalance
settings. Extensive experiments demonstrate the effectiveness of our approach,
showing clear improvements over existing baselines on both image classification and real-world Large Language Model (LLM) routing tasks.

\end{abstract}

\section{Introduction}
\label{sec:intro}

Effective learning algorithms often benefit from routing complex or
ambiguous inputs to specialized experts.  These experts may be
particularly effective in handling intricate domains with overlapping
class boundaries or in resolving uncertainty near decision thresholds.
They can range from human specialists with deep domain expertise to
sophisticated yet computationally intensive machine learning models.
The challenge lies in dynamically assigning each input to the most
suitable expert while balancing the trade-off between accuracy and
computational cost. This becomes even more involved when working with
a diverse set of experts, each with distinct capabilities and
limitations.

The fundamental challenge is to learn from labeled 
examples to route input instances to the most appropriate experts, a
problem known as that of \emph{learning to defer with multiple
experts}.  In natural language generation, tackling this issue is
essential for reducing errors and hallucinations and improving the
efficiency of large language models (LLMs) \citep{WeiEtAl2022,
  bubeck2023sparks}. Beyond NLP, expert selection strategies are
equally critical in domains such as medical diagnosis, image
annotation, economic forecasting, and computer vision, where selecting
the right expert can significantly enhance both accuracy and resource
efficiency.

The \emph{single-stage learning to defer} paradigm has been
extensively studied, beginning with foundational research on learning
with abstention by \citet{CortesDeSalvoMohri2016,
  CortesDeSalvoMohri2016bis, CortesDeSalvoMohri2024}, and followed
by extensive work on abstention and deferral \citep{madras2018learning,
  raghu2019algorithmic, mozannar2020consistent, wilder2021learning,
  pradier2021preferential, keswani2021towards, raman2021improving,
  liu2022incorporating, verma2022calibrated, charusaie2022sample,
  caogeneralizing, verma2023learning, MaoMohriZhong2024deferral,
  MaoMohriZhong2024predictor,MaoMohriZhong2024score,maorealizable,
  pmlr-v206-mozannar23a}. In this single-stage approach, a predictor
and a deferral function are learned jointly, with the deferral
function determining the best expert for each input.

However, in many practical scenarios, strong predictors, such as a
family of LLMs, are already available, and retraining them alongside a
deferral function can be computationally prohibitive. In addition, the models, whether LLMs or other classifiers, may be trained on privacy-sensitive data and only the final models are available for general use.  Thus, the
single-stage learning to defer framework and its associated methods
often overlook the practical constraints encountered in real-world
applications.
To address these limitations, \citet{mao2023two} introduced and
studied the \emph{two-stage learning to defer} framework, where the
family of predictors is fixed and only the deferral function is
learned. They provided non-asymptotic learning guarantees and
effective algorithms, demonstrating strong empirical
performance. \ignore{ Specifically, they developed a novel family of
  surrogate loss functions and algorithms with broad potential
  application, especially in LLMs and other practical settings.  They
  proved that these surrogate losses satisfy $\sH$-consistency bounds
  \citep{awasthi2022h,awasthi2022multi}, which are non-asymptotic,
  hypothesis-set-specific upper bounds on the target estimation loss
  expressed in terms of the surrogate estimation loss.  These bounds
  provide stronger and more informative guarantees than
  Bayes-consistency
  \citep{Zhang2003,bartlett2006convexity,steinwart2007compare,
    mozannar2020consistent}, which only guarantees that minimizing the
  surrogate loss over all measurable functions asymptotically
  minimizes the target loss.}

Nevertheless, learning to defer often faces a significant additional
challenge: in many real-world settings, a small subset of experts is
disproportionately favored across most instances, resulting in
\emph{expert imbalance}.  This issue arises in various contexts,
including LLM-based deferral
\citep{MohriAndorChoiCollinsMaoZhong2023learning} and top-$k$
prediction tasks \citep{cortes2024cardinality}. As a consequence,
deferral algorithms may overlook less frequently applicable, highly
specialized, though possibly costly, experts, and as a result perform only marginally better
than na\"ive baselines that default to the majority expert, see
Appendix~\ref{app:imbalance} for a further discussion and shortcoming
of current approaches. Thus, imbalance may hinder the learning process
and can lead to a degraded overall performance with only small resource gains. Can we design new imbalance-aware deferral
algorithms that outperform the best existing methods, yet keep resource needs low?

\ignore{
Nevertheless, learning to defer often faces a significant additional 
challenge: imbalanced training data. In many real-world scenarios, a
small subset of experts is favored for the majority of instances,
leading to biased learning. This issue arises in various contexts,
including LLM-based deferral
\citep{MohriAndorChoiCollinsMaoZhong2023learning} and top-$k$
prediction tasks \citep{cortes2024cardinality}. As a consequence,
deferral algorithms may 
overlook less frequently applicable, highly specialized experts, and as a result perform only marginally
better than na\"ive baselines that default to the majority expert, see Appendix~\ref{app:imbalance} for a further discussion  and shortcoming of current approaches. Thus,
imbalance may hinder the learning process and can lead to a degraded overall
performance as well as an increase in the cost or latency. Can we design new imbalance-aware deferral algorithms
that outperform the best existing methods?
}

In the multi-class setting, a common strategy for addressing \ignore{data }imbalance involves oversampling
underrepresented classes or undersampling dominant ones
\citep{chawla2002smote,WallaceSmallBrodleyTrikalinos2011,KubatMatwin1997,QiaoLiu2008,han2005borderline,
  estabrooks2004multiple,
  liu2008exploratory,zhang2021learning}. Another related approach
assigns different loss penalties to different classes, in the hope of making the learning algorithm select the specialized experts more often (see
Appendix~\ref{app:related-work}). However, these methods lack strong
theoretical justification, as they modify the training distribution to diverge from the true target distribution. Empirically,
their effectiveness is inconsistent and often depends on extensive
hyperparameter tuning \citep{VanHulse:2007}. In the deferral setting,
such techniques are even more problematic since they would require
assigning additional costs to experts, while the deferral problem
already incorporates instance-specific expert costs.

\begin{figure}[b]
\vskip -0.2in
  \centering
  \includegraphics[scale=0.4]{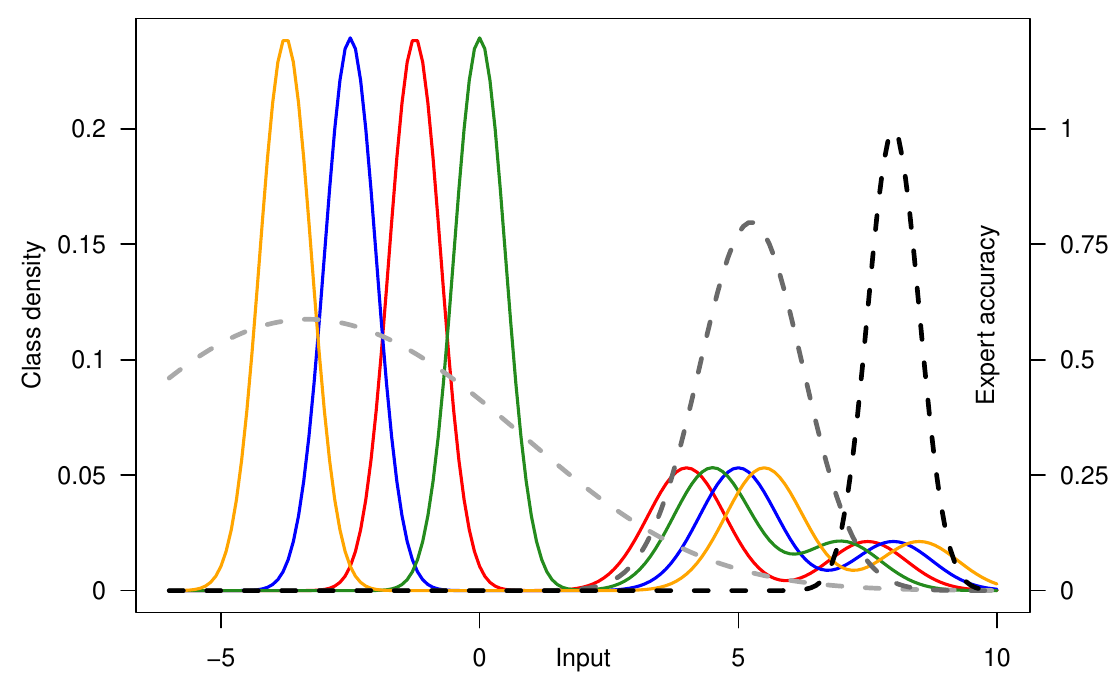}
  \caption{1-D example with 4 classes (colored densities) and 3
    experts (gray lines indicating accuracies). Experts have
    increasing costs indicated by darker shades of gray.}
  \label{fig:4cl3ex}
  \vskip -0.1in
\end{figure}


A further challenge unique to the deferral problem is that the
learning distribution is defined over input-label pairs, whereas the
imbalance we aim to correct concerns \ignore{the distribution of }experts, not class
labels. A 1-D example is provided in Figure~\ref{fig:4cl3ex}
illustrating in colors the densities of 4 classes and accuracies of 3
experts, depicted in shades of gray. The cost of the experts varies,
with the darkest gray representing the most accurate but also the most
expensive expert. In this setting, the left-most expert achieves the
highest accuracy over most of the input distribution. The task of the
two-stage deferral learning problem is to determine which regions of
the input should be deferred to which experts to achieve an optimal
trade-off between cost and accuracy. As illustrated in the figure, this choice may be independent of class labels. This distinction between labels and experts complicates the direct
application of traditional imbalance-handling techniques. Can we
design a principled algorithm for deferral that effectively accounts
for expert imbalance while preserving theoretical soundness and
guarantees?

\textbf{Our contributions.}
We present a detailed study of two-stage learning to defer in
imbalanced settings. \ignore{Appendix~\ref{app:imbalance} provides a detailed discussion of the notion of imbalance in this setting, as well as the potential failure of standard deferral algorithms. To address this issue,} We frame deferral loss minimization as a
cost-sensitive learning problem over the input-expert domain,
introducing a new distribution over input-expert pairs
(Section~\ref{sec:deferral-loss}). This leads us to study
cost-sensitive multi-class classification under imbalance. In
Section~\ref{sec:imbalanced-csl}, we build on recent margin-based
methods to introduce new cost-sensitive margin losses and establish
guarantees tailored for imbalance. We propose novel algorithms
grounded in this theory, which also improve solutions in balanced
structured prediction tasks. In Section~\ref{sec:algorithm},
we design deferral algorithms leveraging class-independent cost
structures, prove strong hypothesis-dependent consistency guarantees,
and unify input-expert optimization with standard input-label
frameworks. Finally, in Section~\ref{sec:experiments}\ignore{ and Appendix~\ref{app:experiments}}, we evaluate \MILD\ against the baseline \citep{mao2023two} across CIFAR-10, CIFAR-100, SVHN, and Tiny ImageNet. We further validate our framework on a real-world LLM routing task (MMLU), showing that \MILD\ effectively reduces computational costs by deferring to lightweight models (0.5B/1.5B) when appropriate, thereby overcoming the tendency of baselines to overuse expensive generalist models.

\section{Preliminaries}
\label{sec:pre}

We consider a standard multi-class classification setting, with an
input space $\sX$ and a label space $\sY = [c] \coloneqq \curl*{1,
\ldots, c}$, where $c$ is the number of classes. We consider the general stochastic setting where a joint distribution $\sD$ over $\sX \times \sY$ allows each $x$ to have multiple possible labels. The expectation $\E_{y \mid x}$ is taken over the conditional distribution $\sD_{y \mid x}$, which introduces randomness into the problem.

We study the \emph{two-stage
learning to defer} framework with multiple experts.
In this setting, we are given a set of $\num \geq 2$ experts,
$\expert_1, \dots, \expert_\num$. In our setting, all experts are fixed pre-trained models, and only the router $f$ is learned. Each expert is represented as a function mapping
$\sX \times \sY$ to $\Rset$.  The learner's goal is to select the most
suitable expert $\expert_k$ for each input, balancing expert accuracy
and inference cost.

Formally, let $\sF$ be a hypothesis set of functions mapping $\sX
\times [\num]$ to $\Rset$, and let $\sF_{\rm{all}}$ be the
set of all measurable functions with the same domain and range. The
goal is to find a predictor $f \in \sF$ that minimizes the
\emph{deferral loss function}, $\tdef$, defined for any $f \in
\sF_{\rm{all}}$ and $(x, y) \in \sX \times \sY$ by:
\begin{equation*}
\tdef(f, x, y)
= \sum_{k = 1}^{\num} c_k(x, y) 1_{\ff(x) = k},
\end{equation*}
where $\ff(x) = \argmax_{k \in [\num]} f(x, k)$ represents the 
prediction of an expert by $f$ for input $x$, using the highest index
for tie-breaking. The cost $c_k(x, y)$ is incurred when expert
$\expert_k$ is selected.

A common choice for $c_k$ incorporates $\expert_k$'s classification
error and inference cost (typically computational) \citep{mao2023two};
for example, $c_k(x, y) = 1_{\expertexpert_k(x) \neq y} + \beta_k$,
where $\expertexpert_k(x) = \argmax_{y \in [c]} \expert_k(x, y)$
represents the prediction made by expert $\expert_k$ for input $x$ and
$\beta_k$ accounts for the inference cost of expert $\expert_k$.

The deferral generalization error of a hypothesis $f$, denoted by
$\sE_{\tdef}(f)$, is the expected deferral loss of $f$ over the data
distribution $\sD$: $\sE_{\tdef}(f) = \E_{(x, y) \sim \sD}
\bracket*{\tdef(f, x, y)}$.  The best-in-class generalization error of
$\sF$, denoted by $\sE_{\tdef}^*(\sF)$, is the infimum of the
generalization errors over all hypotheses in $\sF$:
$\sE^*_{\tdef}(\sF) = \inf_{f \in \sF} \sE_{\tdef}(f)$.  We will adopt
similar definitions for other loss functions.

\section{Deferral Loss Function}
\label{sec:deferral-loss}

In this section, we formulate the minimization of the deferral loss as
a cost-sensitive learning problem over the input-expert domain. To
achieve this, we first derive alternative expressions for
$\ldef$. Recall that the \emph{margin of $f \in \sF$} for the labeled
pair $(x, k) \in \sX \times [\num]$ is defined as $\rho_f(x, k) = f(x,
k) - \max_{k' \neq k} f(x, k')$, which represents the difference of
the score assigned by $f$ to the pair $(x, k)$ and that of the
runner-up expert. The following lemma provides a margin-based reformulation
of the deferral loss.

\begin{restatable}{lemma}{DeferralLossOne}
\label{lemma:deferral-loss-1}
For any $f \in \sF$ and $(x, y) \in \sX \times \sY$, the loss function
$\sfL_{\rm def}$ can be expressed as follows: 
\ifdim\columnwidth=\textwidth
{
\begin{equation}
\sfL_{\rm def}(f, x, y)
=   \sum_{k = 1}^{p} \paren*{\sum_{k' = 1}^p c_{k'}(x, y)1_{k' \neq k}}
  1_{\rho_f(x, k) \leq 0}
  - (p - 2) \sum_{k = 1}^p c_k(x, y).
\end{equation}
}\else
{
\begin{multline*}
\sfL_{\rm def}(f, x, y)\\ 
= \sum_{k = 1}^{p} \paren*{\sum_{k' = 1}^p c_{k'}(x, y)1_{k' \neq k}}
  1_{\rho_f(x, k) \leq 0}
  - (p - 2) \sum_{k = 1}^p c_k(x, y).    
\end{multline*}
}\fi
\end{restatable}
Assuming that the costs satisfy $c_k \in [0, 1]$, which can be
achieved through appropriate normalization, then the deferral loss
function takes the following general form.

\begin{restatable}{lemma}{DeferralLossTwo}
\label{lemma:deferral-loss-2}
For any $f \in \sF$ and $(x, y) \in \sX \times \sY$, the loss function
$\sfL_{\rm def}$ can be expressed as follows: 
\ifdim\columnwidth=\textwidth
{
\begin{equation}
 \sfL_{\rm def}(f, x, y)
= \sum_{k = 1}^p \paren*{1 - c_k(x, y)} 1_{\rho_f(x, k) \leq 0}
  + \sum_{k = 1}^p c_k(x, y) - (p - 1).
\end{equation}
}\else
{
\begin{multline*}
\sfL_{\rm def}(f, x, y)\\
= \sum_{k = 1}^p \paren*{1 - c_k(x, y)} 1_{\rho_f(x, k) \leq 0}
  + \sum_{k = 1}^p c_k(x, y) - (p - 1).
\end{multline*}
}\fi
\end{restatable}
The proofs for both lemmas are presented in
Appendix~\ref{app:deferral-loss}. Based on these results, and ignoring
constants that do not depend on $f$, the loss $\sfL_{\rm def}$ can be
equivalently written as:
\[
\forall (f, x, y),\quad  \sfL_{\rm def}(f, x, y)
= \sum_{k = 1}^p \ov c_k(x, y) 1_{\rho_f(x, k) \leq 0},
\]
for appropriate \emph{rewards} $\ov c_k$ (with $k \in [\num]$), which vary over the support of the joint distribution $\sD$. In particular, Lemma~\ref{lemma:deferral-loss-1} leads to an equivalent form of $\ldef$ by setting reward $\ov c_k(x, y) = \sum_{k' \neq k} c_{k'}(x, y)$ and Lemma~\ref{lemma:deferral-loss-2} yields an alternative equivalent form with reward $\ov c_k(x, y) = 1 - c_k(x, y)$.

Fix $x \in \sX$
and define
$\sfp(k | x) = \frac{\E_{y | x}[\ov c_k(x, y)]}{C(x)}$, which normalizes the reward into a conditional probability, where $ C(x) = \E_{y |
  x}\bracket*{\sum_{k = 1}^p \ov c_k(x, y)}$ is the normalization factor.
Then, we have
\begin{align}
  \E_{y | x} \bracket*{\sfL_{\rm def}(f, x, y)}
  & = \sum_{k = 1}^p \E_{y | x}[\ov c_k(x, y)] 1_{\rho_f(x, k) \leq 0} \nonumber \\
  & = C(x) \sum_{k = 1}^p \sfp(k | x) 1_{\rho_f(x, k) \leq 0}.
\end{align}
Note that the reward $\overline{c}_k$, and consequently the probability $\sfp(k \mid x)$, is inversely proportional to the cost $c_k$. This aligns with the intuition that samples should be deferred to experts with lower cost.
Let $\ov \sD$ denote the marginal distribution over $\sX$ and define
the distribution $\sP$ over $\sX \times [\num]$ by
$\forall (x, k) \in \sX \times [\num],
 \sP(x, k) = \sfp(k | x) \ov \sD(x)$.
Then, we can express the expected deferral loss as:
\begin{align}
\label{eq:new}
 \E_{(x, y) \sim \sD} \bracket*{\sfL_{\rm def}(f, x, y)} &= \E_{(x, k) \sim \sP} \bracket*{\ell_{\rm def}(f, x, k)} \nonumber \\
 \quad \text{ with}\quad \ell_{\rm def}(f, x, k) &= C(x) 1_{\rho_f(x, k) \leq 0}
\end{align}
where the loss function $\ell_{\rm def}$ is defined for all $(f, x, k)
\in \sF_{\rm{all}} \times \sX \times [\num]$.
Thus, $\ell_{\rm def}$ represents a cost-sensitive loss function over
the input-expert domain, suitable for addressing our expert imbalance problem. This motivates the study of
cost-sensitive multi-class classification with imbalanced data in the
following section.

\section{Imbalanced Cost-Sensitive Multi-Class Classification}
\label{sec:imbalanced-csl}

This section gives a theoretical analysis of imbalanced cost-sensitive
multi-class classification, leveraging recent work of
\citet{cortes2025balancing} on imbalanced margin for multi-class
classification. We first extend existing theoretical tools to the
cost-sensitive instance-dependent case. Then, we present a theoretical analysis of this
framework and derive new algorithms for imbalanced cost-sensitive
multi-class classification.  Our algorithms are novel, even in the
balanced case and can lead to improved structured prediction theory and
algorithm design (see Appendix~\ref{app:csl}). \ignore{Note that while our work builds on \citet{cortes2025balancing}, Appendix~\ref{app:novelty} clarifies the key novel aspects, particularly in comparison to that prior work.}

\subsection{Theoretical Analysis}
\label{sec:theory}

Let $f\colon \sX \times [\num] \to \Rset$ be a scoring function
belonging to the hypothesis set $\sF$. 
We define the cost-sensitive zero-one loss function $\sfL$ as follows:
for all $(f, x, k) \in \sF \times \sX \times [\num]$,
\[
\sfL(f, x, k) = c(x, k, \ff(x)) \, 1_{\rho_f(x, k) \leq 0},
\]
where $c(x, k, \ff(x)) \in [0, 1]$ is a non-negative cost bounded by
one that is vanishing when $k = \ff(x)$. Note that $\ell_{\rm def}$ is
a special case of $\sfL$.

\textbf{A. Cost-sensitive margin loss functions.}  We first introduce
new cost-sensitive instance-dependent margin loss functions. \ignore{ which will play a central
role in our derivation of margin-based guarantees for cost-sensitive
learning, in particular in the imbalanced setting.}

Let $\Phi_{\rho} \colon u \mapsto 
\min \paren*{1, \max \paren*{0, 1
    - u / \rho}}$ denote the $\rho$-margin loss function.
We can upper-bound the cost-sensitive zero-one loss function $\sfL$
as follows:
\begin{align*}
  \sfL(f, x, k)
  & \leq c(x, k, \ff(x)) \Phi_{\rho}\paren*{\rho_f(x, k)}\\
  & = c(x, k, \ff(x)) \Phi_{\rho}\paren[\Big]{f(x, k) - \max_{k' \neq k} f(x, k')}\\
  & = c(x, k, \ff(x)) \Phi_{\rho}\paren*{f(x, k) - f(x, \ff(x))}\\
  & \leq \max_{k' \in [\num]} \curl*{c(x, k, k') \Phi_{\rho}\paren*{f(x, k) - f(x, k')}}.
\end{align*}
The second equality follows from the fact that when $k = \ff(x)$ we
have $c(x, k, \ff(x)) = 0$. Otherwise, for $k \neq \ff(x)$, the
runner-up prediction satisfies $\argmax_{k' \neq k} f(x, k') =
\ff(x)$.

This analysis motivates the definition of the
\emph{cost-sensitive margin loss function} as the function
$\sfL_{\rho} \colon \sF_{\mathrm{all}} \times \sX \times [\num] \to
\Rset$, defined as follows, for any fixed $\rho > 0$:
\begin{equation*}
\sfL_{\rho}(f, x, k)
=  \max_{k' \in [\num]} \curl*{c(x, k, k') \Phi_{\rho}\paren*{f(x, k) - f(x, k')}}.
\end{equation*}
Inspired by the analysis of \citet{cortes2025balancing} for imbalanced
learning in standard multi-class classification, we extend our
definition to the imbalanced setting.  Given a vector of margin
parameters $\brho =
[\rho_k]_{k \in [\num]}$, we introduce the \emph{imbalanced
cost-sensitive margin loss function} as the function $\sfL_{\brho}
\colon \sF_{\mathrm{all}} \times \sX \times [\num] \to \Rset$, defined
by: for all $(f, x, k) \in \sF \times \sX \times [\num]$, 
\begin{equation*}
\sfL_{\brho}(f, x, k) = \max_{k' \in [\num]} \curl[\Big]{c(x, k, k') \Phi_{\rho_k} \paren*{f(x, k)
    - f(x, k')}}.
\end{equation*}

\textbf{B. Margin bounds.}
\label{sec:margin-bound}
We now establish a general margin-based generalization bound, which 
serves as the foundation for deriving new algorithms for imbalanced
cost-sensitive classification.

To capture class-specific variations in confidence margins, we adopt
the definition of \emph{class-sensitive Rademacher complexity from
prior work}, which introduces a distinct confidence margin weight
$\rho_i$ for each class $i$.
Given non-negative confidence margin parameters
$\brho = [\rho_k]_{k \in [\num]}$, the empirical
class-sensitive Rademacher complexity of $\sF$ for a sample $S =
\paren*{x_1, \ldots, x_m}$ is defined as:
\begin{equation*}
  \h \Rad_{S, \brho}(\sF)
  = \frac{1}{m} \E_{\bepsilon}\bracket*{\sup_{f \in \sF} \curl*{
      \sum_{j = 1}^p \sum_{i \in S_j} \sum_{k = 1}^p  \e_{ik}
      \frac{f(x_i, k)}{\rho_j}}},
\end{equation*}
where $S_j$ denotes the subsample of $S$ consisting of points labeled
with $j$, with cardinality $m_{j} = |S_{j}|$, and $\bepsilon =
\paren*{\e_{i k}}_{i, k}$ represents a matrix of independent
Rademacher variables $\e_{ik}$s, each uniformly distributed over
$\curl*{-1, +1}$.
For any integer $m \geq 1$, the class-sensitive Rademacher complexity
of $\sF$ is the expectation of $\h \Rad_{S, \brho}(\sF)$ over all
samples $S$ of size $m$: $\Rad_{m, \brho}(\sF) = \E_{S \sim \sD^m}
\bracket*{\h \Rad_{S, \brho}(\sF)}$.

Using these notions of complexity, we prove the following
imbalanced cost-sensitive margin bound.

\begin{restatable}[Margin bound for imbalanced cost-sensitive
    classification]{theorem}{MarginBound}
\label{thm:margin-bound}
Let $\sF$ be a family of functions mapping from $\sX \times [\num]$ to
$\Rset$, and fix $\brho = [\rho_k]_{k \in [\num]}$. Then, for any
$\delta > 0$, with probability at least $1 - \delta$, each of the
following inequalities holds for all $f \in \sF$:
\begin{align*}
\sE_{\sfL}(f) &\leq \h \sE_{S, \brho}(f) + 4 \sqrt{2p}\, \Rad_{m, \brho}(\sF)
  + \sqrt{\frac{\log \frac{1}{\delta}}{2m}}\\
  \sE_{\sfL}(f) &\leq \h \sE_{S, \brho}(f) + 4 \sqrt{2p}\, \h \Rad_{S, \brho}(\sF)
  + 3 \sqrt{\frac{\log \frac{2}{\delta}}{2m}}.
\end{align*}
\end{restatable}
Our proof (see Appendix~\ref{app:margin-bound}) is similar to that of
\citet{cortes2025balancing}, modulo our adaptation to the
instance-dependent cost-sensitive nature of our notion of margin loss. This formulation requires substantial
adaptation and extends the previous work, which is limited to the standard class-imbalanced setting
with uniform costs. In particular, we establish novel margin bounds based on a refined upper bound involving a
maximum operator and derive new Rademacher complexity bounds for this term using the vector contraction lemma.
Our bounds can be generalized to hold uniformly for all $\brho =
[\rho_k]_{k \in [\num]} \in (0, 1]^p$, at the cost of additional $\log
  \log$-terms\ignore{ $\sum_{j = 1}^p \sqrt{\frac{\log \log_2
        \frac{2}{\rho_{j}}}{m}}$}, using standard proof techniques
  \citep[Theorem~5.9]{MohriRostamizadehTalwalkar2018}.
As with standard margin bounds, these learning guarantees suggest a
trade-off: Increasing $\rho_k$ reduces the complexity term (second
term) but simultaneously increases the empirical imbalanced
cost-sensitive margin loss, $\h \sE_{S, \brho}(f)$, by
imposing stricter confidence margin requirements. Thus, if $f$
maintains a low empirical imbalanced cost-sensitive margin loss even
with relatively large $\rho_k$ values, it admits a strong
generalization error guarantee.

\subsection{Algorithms}
\label{sec:algorithms}

The margin guarantees established in the previous section provide a
foundation for developing new algorithms. We begin by deriving a more
explicit learning guarantee within a broad framework, which we then
use to define a general cost-sensitive learning algorithm for imbalanced data.

\textbf{A. Explicit upper bounds}. To make these guarantees
more explicit, we introduce the following setup.
Given a feature mapping $\Phi \colon \sX \times [\num] \to \Rset^d$,
we can identify $\sX \times [p]$ with a subset of $\Rset^d$, with
$\norm{\Phi(x, k)} \leq \sfX_k$ for all $x \in \sX$ and $\sfX = \max_{k
  \in [p]} \sfX_k$, for some norm $\norm*{ \ \cdot \ }$. We assume
$\sF$ is given by $\sF = \curl*{f \in \ov \sF \colon \norm{f}_* \leq
\ov  \sfF}$, for some appropriate norm $\norm*{\, \cdot \,}_*$ on some
space $\ov \sF$ and $\ov \sfF > 0$. This formulation covers a wide range
of hypothesis sets, including linear, kernel-based, and neural network
models. In particular, it captures the settings of neural networks
with weight matrices constrained by a Frobenius norm bound
\citep{CortesGonzalvoKuznetsovMohriYang2017,
  NeyshaburTomiokaSrebro2015} or a spectral norm complexity constraint
relative to reference weight matrices
\citep{BartlettFosterTelgarsky2017}.
In all of these cases, the empirical class-sensitive Rademacher
complexity can be upper bounded as follows:
\begin{equation}
  \h \Rad_{S, \brho}(\sF)
  \leq \frac{\sqrt{p} \, \sfF }{m} \sqrt{\sum_{j = 1}^p
    \frac{m_j \sfX_j^2}{\rho_j^2}}
  \leq \frac{\sqrt{p} \, \sfF \sX}{m} \sqrt{\sum_{j = 1}^p
    \frac{m_j}{\rho_j^2}},
\end{equation}
\vskip -0.1in
where the complexity term $\sfF$ depends on $\ov \sfF$. Combining this upper bound with Theorem~\ref{thm:margin-bound}
yields the following more explicit guarantee.

\begin{corollary}
\label{cor:margin-bound}
Fix $\brho = [\rho_k]_{k \in [\num]}$, then, for any $\delta > 0$,
with probability at least $1 - \delta$ over the choice of a sample $S$
of size $m$, the following holds for any $f \in \sF$:
\begin{align*}
  \sE_{\sfL}(f) &\leq \h \sE_{S, \brho}(f)
  + \frac{4 \sqrt{2} \, p \sfF}{m} \sqrt{\sum_{j = 1}^p \frac{m_j \sfX_j^2}{\rho_j^2}}
  + 3 \sqrt{\frac{\log \frac{2}{\delta}}{2m}}.
\end{align*}
\end{corollary}
\vskip -0.15in
As with Theorem~\ref{thm:margin-bound}, this bound can be generalized
to hold uniformly for all $\brho = [\rho_k]_{k \in [\num]} \in (0,
1]^p$, at the cost of additional $\log \log$-terms. This generalized
  guarantee provides a basis for designing algorithms 
  choosing $f \in \sF$ and $\brho$ to minimize the bound.

Let $\Psi$ be a decreasing convex function such that $\Phi_{\rho}(x)
\leq \Psi\left(\frac{x}{\rho}\right)$ for all $x \in \Rset$ and $\rho
> 0$. $\Psi$ may be the hinge loss, $\Psi(x) = \max(0, 1 - x)$, or any
member of the broad family of composition-sum (comp-sum) losses
\citep{mao2023cross} defined by $\Psi(x) = \Phi^{\tau}(e^{-x})$, with
$\Phi^\tau$ for $\tau \geq 0$ given by
\[
\Phi^{\tau}(u) =
\begin{cases}
  \frac{1}{1 - \tau} \paren*{(1 + u)^{1 - \tau} - 1}
  & \tau \geq 0, \tau \neq 1 \\
\log(1 + u) & \tau = 1,
\end{cases}
\]
for all $u \geq 0$. This family includes the logistic loss ($\tau = 1$) and the
exponential loss ($\tau = 0$).
Using the fact that $\Phi_{\rho}(t) \leq
\Psi\left(\frac{t}{\rho}\right)$, the guarantee of
Corollary~\ref{cor:margin-bound} and its generalization to a uniform
bound can be expressed as: for any $\delta > 0$, with probability at
least $1 - \delta$, for all $f \in \sF$,
\begin{align*}
\sE_{\sfL}(f)
  & \leq \mspace{-5mu}
    \frac{1}{m} \bracket[\bigg]{ \sum_{j = 1}^p \sum_{i \in S_j} \max_{k' \in [\num]}
    \curl*{c(x_i, j, k') \Psi\paren*{\tfrac{f(x_i, j) - f(x_i, k')}{\rho_j}}}
    \mspace{-5mu} }\\
 & \quad + \frac{4 \sqrt{2} \sfF p}{m} \sqrt{\sum_{j = 1}^p \frac{m_j \sfX_j^2}{\rho_j^2}}
  + O\paren*{ \frac{1}{\sqrt{m}}},
\end{align*}
where the last term accounts for the $\log$-$\log$ terms and the
$\delta$-confidence term.\\
\textbf{B. General cost-sensitive algorithm.} Define $\ov \rho =
\sum_{j = 1}^p \rho_{j}$ and $\ov \sfX = \sum_{j = 1}^p (m_j
\sfX_{j}^2)^{\frac{1}{3}}$.  It is straightforward to show that for
constant $\ov \rho$, the second term of the bound is minimized when
$\frac{\rho_j}{\ov \rho} = \frac{(m_j \sfX_{j}^2)^{\frac{1}{3}}}{\ov \sfX}$ for all $j$.
This leads to the following regularization-based algorithm:
\begin{align}
 & \frac{1}{m} \bracket*{ \sum_{j = 1}^p
  \sum_{i \in S_j} \max_{k' \in [\num]} \curl*{c(x_i, j, k')
    \Psi\paren*{\tfrac{f(x_i, j) - f(x_i, k')}{\rho_j}}} } \nonumber \\
     & \quad + \min_{f \in \ov \sF} \lambda \norm*{f}^2, 
\end{align}
where $\lambda$ and $\rho_j$s are selected via cross-validation, with
$\rho_j$s close to $\frac{(m_j \sfX_{j}^2)^{\frac{1}{3}}}{\ov
  \sfX} \ov \rho$.  For a large number of classes $p$, we can assign the same $\rho_j$ to classes with
smaller representation.
\ignore{limit the search by assigning the same $\rho_j$ to classes with
smaller representation.
, and distinct $\rho_j$s only to the top
classes.}

\section{Algorithms for Expert Imbalance Settings}
\label{sec:algorithm}

In this section, we develop deferral algorithms for expert imbalance settings,
leveraging the results of the previous section.  We first derive
simplified cost-sensitive algorithms for the settings where the costs
depend only on $x$. Next, we establish a strong hypothesis
set-dependent consistency guarantee for the corresponding loss
function.  This result further justifies the proposed algorithm.
Finally, we apply these results to the deferral setting, which leads
to a novel deferral algorithm for expert imbalance settings with favorable
theoretic guarantees.

\subsection{Algorithms for Class-Independent Cost-Sensitive Learning}
\label{sec:deferral-algorithms}

When $c(x, k, k')$ is independent of $(k, k')$ (denoted more
simply as $C(x)$), as is the case in the deferral setting (i.e., $\ell_{\rm def}$ in Eq.~\eqref{eq:new}), choosing $\Psi$ to be the logistic loss,
using the monotonicity of the $\log$
function and upper-bounding the maximum by a sum, we obtain: \ignore{$\sfL_\brho(f, x, k)
 \leq 
C(x) \log\bracket*{1 + \max_{k' \neq k} \exp\paren*{\frac{f(x, k') - f(x, k)}{\rho_k}}} \leq C(x)
  \log\bracket*{\sum_{k' = 1}^p \exp\paren*{\frac{f(x, k') - f(x, k)}{\rho_k}}}.
$}
\begin{align*}
\sfL_\brho(f, x, k) &
 \leq \mspace{-3mu}
C(x) \log\bracket*{1 + \max_{k' \neq k} \exp\paren*{\frac{f(x, k') - f(x,
                      k)}{\rho_k}} \mspace{-4mu}} \\
&\leq C(x)
  \log\bracket*{\sum_{k' = 1}^p \exp\paren*{\frac{f(x, k') - f(x, k)}{\rho_k}}}.
\end{align*}

See Appendix~\ref{app:derivation} for detailed derivations and formulations with
other choices of $\Phi$, such as the comp-sum loss with $\tau \neq
1$. We consider the imbalanced cost-sensitive surrogate loss $\wt
\sfL_\brho$ defined as:
\vskip -0.4cm
\begin{align}
   \wt \sfL_\brho(f, x, k) &= C(x) \log\bracket*{
      \sum_{k' = 1}^p \exp\paren*{\frac{f(x, k') - f(x, k)}{\rho_k}}} \nonumber \\
  &\coloneqq C(x) \, \wt \ell_\brho(f, x, k).
\end{align}
Then, in this special case where $c(x, k, k') = C(x)$, the algorithm
presented in the previous section can be re-expressed as follows,
with $\rho_j$s chosen via cross-validation and close to $\frac{(m_j \sfX_{j}^2)^{\frac{1}{3}}}{\ov
  \sfX} \ov \rho$:
  \vskip -0.5cm
\begin{equation}
\min_{f \in \ov \sF} \lambda \norm*{f}^2 + \frac{1}{m}
\sum_{i \in S} C(x_i){\wt \ell_\brho(f, x_i, k_i)}.
\end{equation}

\subsection{Hypothesis-Set Dependent Consistency Bounds}
\label{sec:deferral-bounds}

Given a hypothesis set $\sF$, an \emph{$\sF$-consistency bound}
\citep{awasthi2022h,awasthi2022multi,mao2023cross} for a surrogate
loss $\sfL_1$ of a target loss function $\sfL_2$ is an inequality of
the form
\begin{align}
\forall f \in \sF, \
& \sE_{\sfL_2}(f) - \sE^*_{\sfL_2}(\sF) + \sM_{\sfL_1}(\sF) \nonumber \\ 
& \quad \leq \Gamma\paren*{\sE_{\sfL_1}(f) - \sE^*_{\sfL_1}(\sF) + \sM_{\sfL_1}(\sF)},
\end{align}
where $\Gamma\colon \Rset_+ \to \Rset_+$ is a non-decreasing concave
function with $\Gamma(0) = 0$, and $\sM_{\sfL}(\sF)$ is the
\emph{minimizability gap} for hypothesis set $\sF$ and loss function
$\sfL$. The minimizability gap is defined as the difference between
the best-in-class expected loss and that of the expected pointwise
infimum loss: $\sM_{\sfL}(\sF) = \sE^*_{\sfL}(\sF) - \E_{x} \bracket*
{\inf_{f \in \sF} \E_{y|x}\bracket*{\sfL(f, x, k)}}$. Due to the
super-additivity of the infimum, the minimizability gap is always
non-negative. It becomes zero when the best-in-class error
$\sE^*_{\sfL}(\sF)$ equals the Bayes error
$\sE^*_{\sfL}(\sF_{\rm{all}})$, specifically when $\sF =
\sF_{\rm{all}}$ \citep{mao2024universal}. The $\sF$-consistency bound
relates the minimization of the estimation error for the surrogate
loss $\sfL_1$ to that of the target loss $\sfL_2$ quantitatively.  It
provides a stronger and more informative guarantee than
Bayes-consistency \citep{Zhang2003,bartlett2006convexity,
  steinwart2007compare}, which it implies (by setting $\sF =
\sF_{\rm{all}}$). Bayes-consistency is a fundamental guarantee in the study of surrogate losses, including in learning to defer settings \citep{mozannar2020consistent,verma2022calibrated}, where it ensures that minimizing the excess error of a surrogate loss also minimizes that of the target deferral loss. However, it can be uninformative in practice, as it applies to all measurable functions and ignores the constraints of restricted hypothesis classes. Recent work by \citet{mao2023two} studies $\sF$-consistency bounds for learning to defer, which are more informative because they are specific to the hypothesis class and non-asymptotic.  Note that we use the term $\sF$-consistency as our hypothesis set is denoted by $\sF$.

The following result establishes an $\sF$-consistency bound for the
imbalanced cost-sensitive surrogate loss $\wt \sfL_{\brho}$ introduced
with respect to the cost-sensitive zero-one loss.  A hypothesis set
$\sF$ is considered complete if, for every input-expert pair $(x, k)
\in \sX \times [\num]$, the set of scores $\curl*{f(x, k) \colon f \in
  \sF}$ spans all real numbers.  Most commonly used hypothesis sets
are complete.

\begin{restatable}[$\sF$-consistency bound for imbalanced cost-sensitive
    surrogate loss]{theorem}{HConsistencySurrogate}
\label{thm:H-consistency-surrogate}
Let $\sF$ be a complete hypothesis set. Then, for all $f \in \sF$,
$\brho > \mathbf{0}$, the following inequality holds:
\vskip -0.4cm
\ifdim\columnwidth=\textwidth 
{
\begin{equation*}
  \sE_{\sfL}(f) - \sE^*_{\sfL}(\sF) + \sM_{\sfL}(\sF)
  \leq \sqrt{2} \paren*{\sE_{\wt \sfL_{\brho}}(f) - \sE^*_{\wt \sfL_{\brho}}(\sF) + \sM_{\wt \sfL_{\brho}}(\sF)}^{\frac12}.
\end{equation*}
}\else
{
\begin{multline*}
\sE_{\sfL}(f) - \sE^*_{\sfL}(\sF) + \sM_{\sfL}(\sF)\\
\leq \sqrt{2} \paren*{\sE_{\wt \sfL_{\brho}}(f) - \sE^*_{\wt \sfL_{\brho}}(\sF) + \sM_{\wt \sfL_{\brho}}(\sF)}^{\frac12}.
\end{multline*}
}\fi
\end{restatable}
The proof can be found in Appendix~\ref{app:H-consistency}.  
Theorem~\ref{thm:H-consistency-surrogate}
provides the first $\sF$-consistency guarantee for a
cost-sensitive surrogate loss. Even in the standard imbalanced case
where the cost $C(x) \equiv 1$, it offers new guarantees
for the surrogate loss studied in
\citep{cortes2025balancing}. In this special case,
Theorem~\ref{thm:H-consistency-surrogate} also extends the standard
$\sF$-consistency guarantees in \citep{mao2023cross} to the imbalanced
setting.

\begin{table*}[t]
\vskip -0.1in
  \caption{\textbf{Synthetic experts.} Comparison of our \MILD\ algorithm with 
    \TDEF\ on CIFAR-10, CIFAR-100, SVHN and Tiny ImageNet: \textbf{(a) First cost type} (error rate), \textbf{(b) Second cost type} (error rate + cost).}
\vskip -0.3cm
  \centering
  \begin{tabular}{@{\hspace{0cm}}cc@{\hspace{0cm}}}
     \small{(a)} & \small{(b)}\\
  \resizebox{\columnwidth}{!}{
    \begin{tabular}{@{\hspace{0cm}}lcllccccc@{\hspace{0cm}}}
    \toprule
    \multirow{3}{*}{Method} & \multirow{3}{*}{Setup}& \multirow{3}{*}{Dataset}& {Deferral Loss}& \multicolumn{5}{c}{Ratio of Expert Deferral (\%)}\\
    \cmidrule{5-9}
    & & &DL$\sim$ error rate&1&2&3&4&5\\
    \midrule
    \TDEF & \multirow{9}{*}{(I)} & \multirow{2}{*}{CIFAR-10} & 0.0520 $\pm$ 0.0022 & 71.31 & 19.18 & 9.51 & \NA & \NA \\
    \MILD & & & \textbf{0.0403 $\pm$ 0.0018} & 70.01 & 19.99 & 10.00 & \NA & \NA \\
    \cmidrule{1-1}  \cmidrule{3-9}
    \TDEF & &  \multirow{2}{*}{CIFAR-100} & 0.2399 $\pm$ 0.0019 & 83.26 & 12.54 & 4.20 & \NA & \NA\\
    \MILD & & & \textbf{0.2272 $\pm$ 0.0037} & 81.21 & 12.76 & 6.03 & \NA & \NA \\
    \cmidrule{1-1} \cmidrule{3-9}
    \TDEF & & \multirow{2}{*}{SVHN} & 0.0468 $\pm$ 0.0015 & 83.27 & 12.01 & 4.72 & \NA & \NA \\
    \MILD & & & \textbf{0.0254 $\pm$ 0.0016} & 80.33 & 13.39 & 6.28 & \NA & \NA\\
    \cmidrule{1-1} \cmidrule{3-9}
    \TDEF & & Tiny 
    
    & 0.3488 $\pm$ 0.0028  & 71.80 & 28.16 & 0.04 & \NA & \NA \\
    \MILD & & ImageNet& \textbf{0.3365 $\pm$ 0.0033} & 70.68 & 19.25 & 10.07 & \NA & \NA\\
    \midrule
    \TDEF & \multirow{9}{*}{(II)} & \multirow{2}{*}{CIFAR-10} & 0.0924 $\pm$ 0.0046 & 51.54 & 19.53 & 18.76 & 10.17 & \NA \\
    \MILD & & & \textbf{0.0847 $\pm$ 0.0038} & 51.50 & 19.41 & 18.99 & 10.10 & \NA\\
    \cmidrule{1-1}  \cmidrule{3-9}
    \TDEF & &  \multirow{2}{*}{CIFAR-100} & 0.2982 $\pm$ 0.0028 & 53.71 & 18.32 & 20.54 & 7.43 & \NA \\
    \MILD & & & \textbf{0.2899 $\pm$ 0.0019} & 55.21 & 17.78 & 18.53 & 8.48 & \NA\\
    \cmidrule{1-1} \cmidrule{3-9}
    \TDEF & & \multirow{2}{*}{SVHN} & 0.0604 $\pm$ 0.0027 & 63.47 & 14.87 & 14.76 & 6.91 & \NA \\
    \MILD & & & \textbf{0.0342 $\pm$ 0.0018} & 63.99 & 14.56 & 14.07 & 7.38 & \NA \\
    \cmidrule{1-1} \cmidrule{3-9}
    \TDEF & & Tiny & 0.5287 $\pm$ 0.0032 & 47.81 & 35.10 & 14.78 & 2.31 & \NA \\
    \MILD & & ImageNet& \textbf{0.5072 $\pm$ 0.0036} & 54.69 & 12.23 & 14.82 & 18.26 & \NA\\
    \midrule
    \TDEF & \multirow{9}{*}{(III)} & \multirow{2}{*}{CIFAR-10} & 0.1062 $\pm$ 0.0017  & 38.94 & 20.93 & 20.90 & 9.87 & 9.36 \\
    \MILD & & & \textbf{0.0903 $\pm$ 0.0019} & 40.40 & 20.13 & 19.94 & 10.73 & 8.80\\
    \cmidrule{1-1}  \cmidrule{3-9}
    \TDEF & &  \multirow{2}{*}{CIFAR-100} & 0.3215 $\pm$ 0.0023 & 46.35 & 17.23 & 20.11 & 8.99 & 7.32 \\
    \MILD & & & \textbf{0.3128 $\pm$ 0.0032} & 42.69 & 20.64 & 19.32 & 10.02 & 7.33 \\
    \cmidrule{1-1} \cmidrule{3-9}
    \TDEF & & \multirow{2}{*}{SVHN} & 0.0684 $\pm$ 0.0019 & 53.47 & 17.47 & 15.75 & 7.72 & 5.59 \\
    \MILD & & & \textbf{0.0353 $\pm$ 0.0020} & 54.12 & 18.46 & 14.75 & 6.17 & 6.50\\
    \cmidrule{1-1} \cmidrule{3-9}
    \TDEF & & Tiny  & 0.5857 $\pm$ 0.0038  & 36.38 & 17.02 & 36.02 & 9.95 & 0.63 \\
    \MILD & & ImageNet & \textbf{0.5656 $\pm$ 0.0029} & 46.09 & 18.91 & 31.29 & 3.24 & 0.47 \\
    \bottomrule
    \end{tabular}
    } & 
    \resizebox{\columnwidth}{!}{
    \begin{tabular}{@{\hspace{0cm}}lcllccccc@{\hspace{0cm}}}
    \toprule
    \multirow{3}{*}{Method} & \multirow{3}{*}{Setup} & \multirow{3}{*}{Dataset} & {Deferral Loss}
    & \multicolumn{5}{c}{Ratio of Expert Deferral (\%)}\\
    \cmidrule{5-9}
    & & &DL$\sim$ error + cost&1 &2&3&4&5\\
    \midrule
    \TDEF & \multirow{9}{*}{(I)} & \multirow{2}{*}{CIFAR-10} & 0.5950 $\pm$ 0.0011 & 63.68 & 23.61 & 12.71 & \NA & \NA \\
    \MILD & & & \textbf{0.5779 $\pm$ 0.0018} & 67.02 & 20.04 & 12.94 & \NA & \NA \\
    \cmidrule{1-1}  \cmidrule{3-9}
    \TDEF & &  \multirow{2}{*}{CIFAR-100} & 0.8150 $\pm$ 0.0037 & 57.13 & 23.66 & 19.21 & \NA & \NA\\
    \MILD & & & \textbf{0.7928 $\pm$ 0.0032} & 50.43 & 27.27 & 22.30 & \NA & \NA \\
    \cmidrule{1-1} \cmidrule{3-9}
    \TDEF & & \multirow{2}{*}{SVHN} & 0.6285 $\pm$ 0.0026 & 77.75 & 14.71 & 7.54 & \NA & \NA \\
    \MILD & & & \textbf{0.6170 $\pm$ 0.0024} & 78.33 & 14.58 & 7.09 & \NA & \NA\\
    \cmidrule{1-1} \cmidrule{3-9}
    \TDEF & & Tiny  & 0.8819 $\pm$ 0.0016 & 4.57 & 75.31 & 20.12 & \NA & \NA \\
    \MILD & & ImageNet& \textbf{0.8653 $\pm$ 0.0019} & 8.49 & 43.39 & 48.12 & \NA & \NA\\
    \midrule
    \TDEF & \multirow{9}{*}{(II)} & \multirow{2}{*}{CIFAR-10} &  0.4421 $\pm$ 0.0034 & 48.11 & 22.81 & 18.31 & 10.77 & \NA \\
    \MILD & & & \textbf{0.4240 $\pm$ 0.0021} & 44.87 & 22.65 & 20.50 & 11.98 & \NA\\
    \cmidrule{1-1}  \cmidrule{3-9}
    \TDEF & &  \multirow{2}{*}{CIFAR-100} &  0.6687 $\pm$ 0.0044 & 42.92 & 23.50 & 23.31 & 10.27 & \NA \\
    \MILD & & & \textbf{0.6506 $\pm$ 0.0032} & 45.04 & 20.86 & 21.04 & 13.06 & \NA\\
    \cmidrule{1-1} \cmidrule{3-9}
    \TDEF & & \multirow{2}{*}{SVHN} & 0.4265 $\pm$ 0.0017 & 60.79 & 15.50 & 15.85 & 7.86 & \NA \\
    \MILD & & & \textbf{0.4148 $\pm$ 0.0023} & 61.04 & 17.86 & 14.04 & 7.06 & \NA \\
    \cmidrule{1-1} \cmidrule{3-9}
    \TDEF & & Tiny  & 0.8576 $\pm$ 0.0028 & 32.27 & 16.68 & 18.72 & 32.33 & \NA \\
    \MILD & & ImageNet& \textbf{0.8324 $\pm$ 0.0019} & 28.91 & 7.27 & 24.55 & 39.27 & \NA\\
    \midrule
    \TDEF & \multirow{9}{*}{(III)} & \multirow{2}{*}{CIFAR-10} & 0.3684 $\pm$ 0.0013 & 37.53 & 22.33 & 18.14 & 11.75 & 10.25 \\
    \MILD & & & \textbf{0.3512 $\pm$ 0.0015} & 37.55 & 20.26 & 20.48 & 10.09 & 11.62\\
    \cmidrule{1-1}  \cmidrule{3-9}
    \TDEF & &  \multirow{2}{*}{CIFAR-100} & 0.6051 $\pm$ 0.0055 & 37.46 & 18.63 & 20.67 & 12.77 & 10.47 \\
    \MILD & & & \textbf{0.5859 $\pm$ 0.0047} & 32.23 & 21.34 & 21.39 & 13.46 & 11.58 \\
    \cmidrule{1-1} \cmidrule{3-9}
    \TDEF & & \multirow{2}{*}{SVHN} & 0.3412 $\pm$ 0.0031 & 52.13 & 18.02 & 15.93 & 6.39 & 7.53 \\
    \MILD & & & \textbf{0.3290 $\pm$ 0.0022} & 52.22 & 19.02 & 15.82 & 6.63 & 6.31 \\
    \cmidrule{1-1} \cmidrule{3-9}
    \TDEF & & Tiny & 0.8481 $\pm$ 0.0035  & 33.65 & 42.02 & 8.92 & 8.82 & 6.59 \\
    \MILD & & ImageNet & \textbf{0.8167 $\pm$ 0.0031} & 17.22 & 10.02 & 15.48 & 35.43 & 21.85 \\
    \bottomrule
    \end{tabular}
    }\\
  \end{tabular}
    \label{tab:comparison}
    \vskip -0.08in
\end{table*}

\begin{table*}[t]
\vskip -0.1in
    \caption{\textbf{Real experts.} Comparison of our \MILD\ algorithm with
    \TDEF\ on CIFAR-10, CIFAR-100, SVHN with real experts: \textbf{(a) First cost type} (error rate), \textbf{(b) Second cost type} (error rate + cost).}
    \vskip -0.3cm
    \centering
    \begin{tabular}{@{\hspace{0cm}}cc@{\hspace{0cm}}}
    \small{(a)} & \small{(b)}\\
    \resizebox{\columnwidth}{!}{
    \begin{tabular}{@{\hspace{0cm}}lcllccccc@{\hspace{0cm}}}
    \toprule
    \multirow{3}{*}{Method} & \multirow{3}{*}{Setup} & \multirow{3}{*}{Dataset} & Deferral Loss
    & \multicolumn{5}{c}{Ratio of Expert Deferral (\%)} \\
    \cmidrule{5-9}
    & & & DL $\sim$ error rate & 1 & 2 & 3 & 4 & 5\\
    \midrule
    \TDEF & \multirow{7}{*}{(I)} & \multirow{2}{*}{CIFAR-10} & 0.1170 $\pm$ 0.0027 & 62.49 & 25.77 & 11.74 & \NA & \NA \\
    \MILD & & & \textbf{0.1060 $\pm$ 0.0020} & 66.87 & 20.75 & 12.38 & \NA & \NA \\
    \cmidrule{1-1}  \cmidrule{3-9}
    \TDEF & &  \multirow{2}{*}{CIFAR-100} & 0.4368 $\pm$ 0.0059 & 87.64 & 7.30 & 5.06 & \NA & \NA\\
    \MILD & & & \textbf{0.4265 $\pm$ 0.0042} & 76.73 & 13.31 & 9.96 & \NA & \NA \\
    \cmidrule{1-1} \cmidrule{3-9}
    \TDEF & & \multirow{2}{*}{SVHN} & 0.0752 $\pm$ 0.0013 & 51.58 & 26.99 & 21.43 & \NA & \NA \\
    \MILD & & & \textbf{0.0579 $\pm$ 0.0015} & 77.16 & 17.39 & 5.45 & \NA & \NA\\
    \midrule
    \TDEF & \multirow{7}{*}{(II)} & \multirow{2}{*}{CIFAR-10} & 0.1480 $\pm$ 0.0020  & 46.95 & 18.64 &  23.08 & 11.33 & \NA \\
    \MILD & & & \textbf{0.1339 $\pm$ 0.0023} & 45.53  & 22.17 &  21.57 & 10.73 & \NA\\
    \cmidrule{1-1}  \cmidrule{3-9}
    \TDEF & &  \multirow{2}{*}{CIFAR-100} &  0.4620 $\pm$ 0.0064 & 63.66 & 13.84 & 13.56  & 8.94 & \NA \\
    \MILD & & & \textbf{0.4375 $\pm$ 0.0048} & 56.15 & 18.07 & 15.76 & 10.02 & \NA\\
    \cmidrule{1-1} \cmidrule{3-9}
    \TDEF & & \multirow{2}{*}{SVHN} & 0.0831 $\pm$ 0.0027 & 32.44 & 23.27 & 26.61  & 17.68 &  \NA \\
    \MILD & & & \textbf{0.0634 $\pm$ 0.0011} & 50.63 & 20.55 & 23.86 & 4.96 &  \NA \\
    \midrule
    \TDEF & \multirow{7}{*}{(III)} & \multirow{2}{*}{CIFAR-10} & 0.1554 $\pm$ 0.0011 & 33.16 & 22.08 & 18.13  & 16.16 & 10.47  \\
    \MILD & & & \textbf{0.1507 $\pm$ 0.0027} & 32.92 & 20.27 & 18.58 & 16.97 & 11.26 \\
    \cmidrule{1-1}  \cmidrule{3-9}
    \TDEF & &  \multirow{2}{*}{CIFAR-100} & 0.4554 $\pm$ 0.0023 & 50.01 & 16.73 & 20.18  & 7.49 & 5.59  \\
    \MILD & & & \textbf{0.4419 $\pm$ 0.0034} & 48.06 & 19.24 & 17.29 & 7.56 & 7.85  \\
    \cmidrule{1-1} \cmidrule{3-9}
    \TDEF & & \multirow{2}{*}{SVHN} & 0.0886 $\pm$ 0.0012 &  27.21 & 19.98 & 20.28 & 11.78 & 20.75   \\
    \MILD & & & \textbf{0.0707 $\pm$ 0.0017}  & 45.78 & 19.31 & 15.85 & 6.73 & 12.33 \\
    \bottomrule
    \end{tabular}
    } &
    \resizebox{\columnwidth}{!}{
    \begin{tabular}{@{\hspace{0cm}}lcllccccc@{\hspace{0cm}}}
    \toprule
    \multirow{3}{*}{Method} & \multirow{3}{*}{Setup} & \multirow{3}{*}{Dataset} & {Deferral Loss }
    & \multicolumn{5}{c}{Ratio of Expert Deferral (\%)} \\
    \cmidrule{5-9}
    & & & DL $\sim$ error + cost& 1 & 2 & 3 & 4 & 5\\
    \midrule
    \TDEF & \multirow{7}{*}{(I)} & \multirow{2}{*}{CIFAR-10} & 0.5413 $\pm$ 0.0036 & 22.93 & 33.11 & 43.96 & \NA & \NA \\
    \MILD & & & \textbf{0.5242 $\pm$ 0.0038} & 10.09 & 23.20 & 66.71 & \NA & \NA \\
    \cmidrule{1-1}  \cmidrule{3-9}
    \TDEF & &  \multirow{2}{*}{CIFAR-100} & 0.9270 $\pm$ 0.0084 & 25.27 & 27.45 & 47.28 & \NA & \NA\\
    \MILD & & & \textbf{0.8941 $\pm$ 0.0042} & 11.70 & 39.13 & 49.18 & \NA & \NA \\
    \cmidrule{1-1} \cmidrule{3-9}
    \TDEF & & \multirow{2}{*}{SVHN} & 0.3259 $\pm$ 0.0017 & 2.06 & 12.04 & 85.90 & \NA & \NA \\
    \MILD & & & \textbf{0.3200 $\pm$ 0.0009} & 5.94 & 23.98 & 70.08 & \NA & \NA\\
    \midrule
    \TDEF & \multirow{7}{*}{(II)} & \multirow{2}{*}{CIFAR-10} &  0.4538 $\pm$ 0.0016  & 26.88 & 19.50 & 23.58 & 30.05 & \NA \\
    \MILD & & & \textbf{0.4432 $\pm$ 0.0022} & 20.90 & 19.41 & 21.12 & 38.57 & \NA\\
    \cmidrule{1-1}  \cmidrule{3-9}
    \TDEF & &  \multirow{2}{*}{CIFAR-100} & 0.7824 $\pm$ 0.0031 & 39.36 & 22.89 & 19.26 & 18.49 & \NA \\
    \MILD & & & \textbf{0.7743 $\pm$ 0.0037} & 30.48 & 22.55 & 22.06 & 24.91 & \NA\\
    \cmidrule{1-1} \cmidrule{3-9}
    \TDEF & & \multirow{2}{*}{SVHN} & 0.2972 $\pm$ 0.0029 & 11.74 & 28.73 & 12.54 & 46.99 &  \NA \\
    \MILD & & & \textbf{0.2802 $\pm$ 0.0027} & 9.88 & 12.26 & 11.40 & 66.46 &  \NA \\
    \midrule
    \TDEF & \multirow{7}{*}{(III)} & \multirow{2}{*}{CIFAR-10} & 0.4121 $\pm$ 0.0009 & 27.39 & 18.78 & 4.24 & 12.58 & 37.01   \\
    \MILD & & & \textbf{0.4022 $\pm$ 0.0011} & 12.72 & 24.22 & 7.99 & 39.51 & 15.56 \\
    \cmidrule{1-1}  \cmidrule{3-9}
    \TDEF & &  \multirow{2}{*}{CIFAR-100} &  0.7781 $\pm$ 0.0053 & 44.55 & 16.41 & 14.67 & 12.03 & 12.34  \\
    \MILD & & & \textbf{0.7639 $\pm$ 0.0046} & 33.41 & 17.93 & 18.45 & 13.40 & 16.81 \\
    \cmidrule{1-1} \cmidrule{3-9}
    \TDEF & & \multirow{2}{*}{SVHN} & 0.2844 $\pm$ 0.0046 &  6.14 & 7.92 & 18.60 & 43.07 & 24.27  \\
    \MILD & & & \textbf{0.2682 $\pm$ 0.0028}  & 12.38 & 15.95 & 18.06 & 25.65 & 27.96 \\
    \bottomrule
    \end{tabular}
    }
    \end{tabular}
    \label{tab:comparison-real}
    \vskip -0.08in
\end{table*}

\subsection{Application to Deferral}
\label{sec:deferral-applicaton}

In the context of deferral, to be precise, we should choose $c(x, j,
k') = C(x) 1_{j = k'}$. However, since $\curl*{j = k'}$ leads to a
constant term, we instead consider the surrogate loss $\wt
\sfL_{\brho}$ with $\Psi$ chosen as the logistic loss.  Alternatively,
we can use other functions $\Psi$, such as the comp-sum loss (see
Appendix~\ref{app:derivation}). By reformulating the input-expert problem within the input-label domain, we obtain:
\begin{align*}
  \E_{(x, k) \sim \sP}[\wt \sfL_{\brho}(f, x, k)]
  & = \E_{x \sim \sD_\sX}\bracket*{\E_{k \sim
    \sfp(\cdot|x)}\bracket*{C(x) \, \wt \ell_\brho(f, x, k)}}\\
  & = \E_{x \sim \sD_\sX}\bracket*{\sum_{k = 1}^p C(x) \sfp(k | x) \wt \ell_\brho(f, x, k)}\\
  & = \E_{x \sim \sD_\sX}\bracket*{\sum_{k = 1}^p \E_{y | x} [\ov
    c_k(x, y)] \, \wt \ell_\brho(f, x, k) \mspace{-4mu}}\\
  & = \E_{(x, y) \sim \sD}\bracket*{\sum_{k = 1}^p \ov c_k(x, y) \, \wt \ell_\brho(f, x, k)}.
\end{align*}
Thus, given a sample $S$ drawn from $\sD^m$, the empirical objective
to minimize for our algorithm is $\E_{(x, y) \sim S}\bracket*{\sum_{k
    = 1}^p \ov c_k(x, y) \, \wt \ell_\brho(f, x, k)}$. This leads to
a novel algorithm for deferral with expert imbalance, defined by: 
\begin{align*}
  & \min_{f \in \ov \sF} \ \lambda \norm{f}^2 + \frac{1}{m} 
  \sum_{i=1}^m\sum_{k = 1}^p  \ov c_k(x_i, y_i)\wt \ell_\brho(f, x_i, k),\\
  & \quad \text{ with} \quad \mspace{-6mu} \wt \ell_\brho(f, x, k) = \log\bracket*{
    \sum_{k' = 1}^p \exp\paren*{\frac{f(x, k') - f(x, k)}{\rho_k}} \mspace{-4mu}}.   
\end{align*}
Let $\wt \sfL_{\rm{def}, \rho}(f, x, y) = \sum_{k = 1}^p \ov c_k(x, y)
\, \wt \ell_\brho(f, x, k)$ be the corresponding deferral surrogate loss. We call our new algorithm \textbf{\MILD\ (\emph{Margin-based Imbalanced Learning to Defer})}.

\textbf{Intuitive interpretation.} Intuitively, \MILD\ dynamically adjusts the confidence threshold required to defer to each expert based on their overall utility and cost. By mathematically penalizing the over-selection of a dominant expert via the margin parameter $\rho_k$, \MILD\ forces the router to require exponentially higher confidence to defer to the most expensive or dominant model. This effectively makes it ``cheaper'' for the router to trust highly specialized or underutilized experts, preventing the system from lazily collapsing to the most common choice.

By reformulating the input-expert problem within the input-label domain, Theorem~\ref{thm:H-consistency-surrogate} directly yields the following $\sF$-consistency bound for the
deferral loss.

\begin{restatable}{corollary}{HConsistencyDeferral}
\label{cor:H-consistency-deferral}
Let $\sF$ be a complete hypothesis set. Then, for all $f \in \sF$,
$\brho > \mathbf{0}$, the following holds:
\ifdim\columnwidth=\textwidth 
{
\begin{equation*}
  \sE_{\ldef}(f) - \sE^*_{\ldef}(\sF) + \sM_{\ldef}(\sF)
  \leq \sqrt{2} \paren*{\sE_{\wt \sfL_{\rm{def}, \rho}}(f)
    - \sE^*_{\wt \sfL_{\rm{def}, \rho}}(\sF)
    + \sM_{\wt \sfL_{\rm{def}, \rho}}(\sF)}^{\frac12}.
\end{equation*}
}\else
{
\begin{multline*}
\sE_{\ldef}(f) - \sE^*_{\ldef}(\sF) + \sM_{\ldef}(\sF)\\
\leq \sqrt{2} \paren*{\sE_{\wt \sfL_{\rm{def}, \rho}}(f)
  - \sE^*_{\wt \sfL_{\rm{def}, \rho}}(\sF)
  + \sM_{\wt \sfL_{\rm{def}, \rho}}(\sF)}^{\frac12}.
\end{multline*}
}\fi
\end{restatable}
As mentioned above, when $\sE^*_{\sfL}(\sF) =
\sE^*_{\sfL}(\sF_{\rm{all}})$, the minimizability gaps vanish. A
special case occurs when $\sF = \sF_{\rm{all}}$. Thus, we further
obtain the following excess error bound for the deferral loss.
\begin{corollary}
\label{cor:H-consistency-deferral-all}
For all $f \in \sF$,
$\brho > \mathbf{0}$, the following excess error bound holds:
\begin{equation*}
\sE_{\ldef}(f) - \sE^*_{\ldef}\paren*{\sF_{\rm{all}}}
  \leq \sqrt{2} \paren*{\sE_{\wt \sfL_{\rm{def}, \rho}}(f) - \sE^*_{\wt \sfL_{\rm{def}, \rho}}\paren*{\sF_{\rm{all}}}}^{\frac12}.
\end{equation*}
\end{corollary}
Corollaries~\ref{cor:H-consistency-deferral} and
\ref{cor:H-consistency-deferral-all} provide strong theoretical
guarantees for the deferral algorithm based on minimizing the deferral
surrogate loss $\wt \sfL_{\rm{def}, \rho}$. In particular, when the
surrogate estimation loss $\sE_{\wt \sfL_{\rm{def}, \rho}}(f) -
\sE^*_{\wt \sfL_{\rm{def}, \rho}}(\sF)$ is reduced to a small value of
$\e$, the target deferral estimation loss $ \sE_{\ldef}(f) -
\sE^*_{\ldef}(\sF)$ is upper bounded by $\sqrt{2 \e}$.

\section{Experiments}
\label{sec:experiments}

We evaluated \MILD\ against the state-of-the-art baseline \TDEF\ \citep{mao2023two} on image classification benchmarks and a Large Language Model (LLM) routing task.
\TDEF\ minimizes a surrogate loss shown to be $\sF$-consistent and serves as the primary baseline, being the only existing method designed specifically for two-stage multi-expert deferral (see Section~\ref{sec:further_analysis} and Appendix~\ref{app:confidence} for the relationship to standard cost-sensitive and confidence-based methods).
Both methods are trained using a logistic surrogate loss. For \MILD, we define the reward $\ov c_k(x, y) = \sum_{k' \neq k} c_{k'}(x, y)$ following Lemma~\ref{lemma:deferral-loss-1}. Choice of $\brho$ in \MILD\ follows the theoretical optima (see Appendix~\ref{app:exp-rho}).

We report \emph{Deferral Loss (DL)}, the average target loss $\tdef(f, x, y)$ on test data, alongside expert deferral ratios. Results are reported as mean $\pm$ standard deviation over five runs. 
For simplicity, we
omit the standard deviations of the deferral ratios. 

\textbf{Image Classification Benchmarks.} We used CIFAR-10, CIFAR-100 \citep{Krizhevsky09learningmultiple}, SVHN \citep{Netzer2011}, and Tiny ImageNet \citep{le2015tiny}. These vision tasks serve as controlled proxies where expert imbalance perfectly mirrors class imbalance, allowing us to mathematically isolate the mechanics of the router. To simulate varying degrees of expert imbalance, we constructed three setups:
\emph{Setup I} (70\%-20\%-10\% coverage), \emph{Setup II} (50\%-20\%-20\%-10\%), and \emph{Setup III} (40\%-20\%-20\%-10\%-10\%). Detailed configurations are provided in Appendix~\ref{app:exp-image}.
We tested two cost functions: (a) \emph{error only}: $c_j = 1_{\expertexpert_{j} \neq y}$; and (b) \emph{error + cost}: penalizing both error and expert coverage. We adopted a ResNet-18 \citep{he2016deep} as the predictor model. Both methods were trained using the Adam optimizer \citep{kingma2014adam} with a batch size of $1,024$, weight decay of $1\times 10^{-3}$, and learning rate of $1 \times 10^{-3}$ for $200$ epochs.

\textbf{Synthetic Expert Results.}
Table~\ref{tab:comparison} presents results for synthetic experts (perfect accuracy on sub-domains). \MILD\ consistently achieves lower deferral loss by allocating experts closer to the optimal distribution.
In the cost-sensitive setting (Table~\ref{tab:comparison}(b)), \MILD\ effectively shifts mass from the high-cost Expert 1 to lower-cost experts, outperforming \TDEF. Note that the higher loss magnitude in setting (b) reflects the average cost of the setting (sum of cost $\times$ coverage), like $0.7^2 +0.2^2 +0.1^2 =0.54$ for Setup~I.

\textbf{Real Expert Results.}
Table~\ref{tab:comparison-real} presents results using experts trained on data subsets (imperfect accuracy). \MILD\ consistently outperforms \TDEF\ across all settings. For example, on CIFAR-100 (Setup I, Error Only), \MILD\ reduces loss from 0.4368 to 0.4265, successfully mitigating the tendency of \TDEF\ to ignore specialized experts. Further experiments with severe expert imbalance are detailed in Appendix~\ref{app:exp-severe}, highlighting the benefits of \MILD\ in extreme scenarios.

\subsection{LLM Routing on MMLU}
\label{sec:exp_llm}

\textbf{LLM Routing on MMLU.} This task serves as our primary natural imbalance benchmark. Here, expert imbalance is entirely independent of class labels; instead, the multiple-choice labels do not dictate expert specialization. The experts' competence varies naturally based on the inherent linguistic complexity of the input features. This natural capability gap causes the 7B model to dominate organically, demonstrating that expert imbalance occurs independently of class labels.

To validate \MILD\ in this context, we evaluate on the MMLU benchmark \citep{hendrycks2020mmlu} using the Qwen 2.5 family \citep{qwen2.5} as experts: \texttt{7B-Instruct} (Strong), \texttt{1.5B-Instruct} (Medium), and \texttt{0.5B-Instruct} (Tiny). We compare \MILD\ against \TDEF\ and three additional baselines adapted for routing (Standard Cross-Entropy (CE), Class-Weighted CE (CWCE), and LDAM \citep{cao2019learning}) on a subset of Mathematics and History tasks under two settings:
(a) \emph{error only}: Costs based purely on error ($c_j = 1_{\expertexpert_{j} \neq y}$). Expert 1 is optimal for 82.2\% of queries.
(b) \emph{error + cost}: Costs include inference penalties ($\boldsymbol \beta = [\beta_1, \beta_2, \beta_3] = [1.0, 0.6, 0.1]$). Efficiency incentives make Expert 3 optimal for 56.8\% of queries.

For the LLM routing task, we adopted a DeBERTa-v3-xsmall \citep{he2021debertav3} model as the router. Both methods were fine-tuned for 5 epochs with a learning rate of $5 \times 10^{-5}$ using the AdamW optimizer \citep{loshchilov2019decoupled}. Extended details are provided in Appendix~\ref{app:exp-LLM}.

Table~\ref{tab:llm_results} shows the results. In setting (a), \TDEF\ collapses to the majority expert (99.6\% usage of Expert 1), ignoring the 17.8\% of cases where smaller models suffice. \MILD\ recovers an imbalanced distribution ($83.3\%$ Strong, $10.4\%$ Mid) much closer to the \textit{Optimal} oracle ($82.2\%$, $12.8\%$).

In setting (b), the contrast is sharper. The \textit{Optimal} strategy uses the Strong expert 31.2\% of the time. \TDEF\ fails to learn this, collapsing to 0\% usage of the Strong expert (selecting only the cheapest). \MILD\ successfully identifies the complex queries, routing 17.9\% to the Strong expert and 73.8\% to Tiny, achieving a loss significantly lower than \TDEF. Furthermore, \MILD\ ($0.813 \pm 0.040$) significantly outperforms the other tested baselines: Standard CE ($0.985 \pm 0.038$), CWCE ($0.942 \pm 0.045$), and LDAM ($0.915 \pm 0.035$). Traditional heuristics either ignore instance-dependent routing costs entirely or rely strictly on global static frequencies, which are fundamentally insufficient for dynamic, cost-sensitive expert deferral.

Estimating the offline cost matrix via a small, fixed routing training dataset is a one-time, computationally negligible step compared to the massive live inference compute savings achieved by the routing system during deployment. For instance, in our MMLU experiment, the router was trained on just 4,000 queries. Once deployed, routing millions of real-world queries relies purely on a single forward pass through the lightweight router and incurs zero extra expert evaluation overhead. Furthermore, our mathematical formulation is model-agnostic: whether an expert is a 7B or a 70B+ parameter model, the router strictly operates on the inferred numerical cost distributions, guaranteeing mathematically identical scaling dynamics.

\begin{table}[t]
\centering
\caption{LLM Routing on MMLU with Qwen 2.5. \MILD\ adapts to the cost structure and tracks the optimal oracle distribution far better than \TDEF, which tends to collapse to extremes.}
\vskip -0.05in
\label{tab:llm_results}
\resizebox{\columnwidth}{!}{
\begin{tabular}{l l c c c c}
\toprule
& & Def. Loss & \multicolumn{3}{c}{Ratio of Expert Deferral (\%)} \\
\cmidrule(lr){4-6}
Setting & Method & (mean $\pm$ std) & 7B (Str) & 1.5B (Mid) & 0.5B (Tiny) \\
\midrule
\multirow{3}{*}{\shortstack[l]{(a) Error Only\\ \scriptsize{($\boldsymbol \beta = \vec{0}$)} }} 
& \textit{Optimal} & \textit{---} & \textit{82.2} & \textit{12.8} & \textit{5.0} \\
& \TDEF & $0.438 \pm 0.080$ & $99.6$ & $0.4$ & $0.0$ \\
& \textbf{\MILD} & $\mathbf{0.425 \pm 0.080}$ & $83.3$ & $10.4$ & $6.2$ \\
\midrule
\multirow{3}{*}{\shortstack[l]{(b) Error + Cost\\ \scriptsize{($\boldsymbol \beta = [1, .6, .1]$)} }} 
& \textit{Optimal} & \textit{---} & \textit{31.2} & \textit{12.0} & \textit{56.8} \\
& \TDEF & $0.928 \pm 0.030$ & $0.0$ & $0.0$ & $100.0$ \\
& \textbf{\MILD} & $\mathbf{0.813 \pm 0.040}$ & $17.9$ & $8.3$ & $73.8$ \\
\bottomrule
\end{tabular}
}
\vskip -0.3in
\end{table}

\subsection{Further Analysis and Discussion}
\label{sec:further_analysis}

\textbf{Ablation Study on \texorpdfstring{$\brho$}{rho} Selection.}
To isolate the empirical contribution of our theoretical derivation, we evaluated \MILD\ on the LLM MMLU routing task (under the ``Error + Cost'' setting) in three configurations: 
1) \emph{\MILD\ (Uniform)}: a na\"ive baseline setting all $\rho_k = 1/3$.
2) \emph{\MILD\ (Theory)}: strictly un-tuned theoretically derived values ($\rho_k \propto (m_k \sfX_k^2)^{1/3}$), entirely skipping cross-validation.
3) \emph{\MILD\ (Tuned)}: tuned via a coarse grid search around the theoretical prior. The theoretical initialization alone reduces the test deferral loss ($0.822 \pm 0.042$) compared to uniform margins ($0.884 \pm 0.055$), nearly matching the fully tuned performance ($0.813 \pm 0.040$). This proves the theoretical bound directly drives the algorithm's success, with cross-validation merely providing minor empirical smoothing.

\textbf{Robustness to Noisy Cost-Matrix Estimates.}
In practice, expert performance characteristics only need to be estimated once offline over a fixed training set. However, to demonstrate that \MILD\ is robust to imperfect or noisy estimations during this offline training phase, we injected severe noise (20\% random Gaussian variations) directly into the estimated cost matrix to simulate highly stochastic and unpredictable expert errors (e.g., LLM hallucinations).

Under this 20\% noise setting, the test deferral loss of \TDEF\ degraded from $0.928 \pm 0.030$ (clean) to $1.054 \pm 0.045$. In contrast, \MILD\ remained highly stable, with its test deferral loss only slightly changing from $0.813 \pm 0.040$ (clean) to $0.826 \pm 0.041$. Its robust margin-based decision boundaries act as a buffer that absorbs stochastic noise, preventing the router from overfitting to erratic estimation errors.

\textbf{Extreme Imbalance in LLM Routing.}
To test the limits on our real-world LLM task, we evaluated an extreme $95\%/5\%$ optimal split on the MMLU dataset (under the ``Error + Cost'' setting) by heavily skewing the inference routing costs to favor the 7B model. Under extreme imbalance, \TDEF\ collapses entirely to 100\% usage of the majority expert ($0.950 \pm 0.030$), failing to identify the remaining 5\% of cases where smaller models suffice. In contrast, \MILD\ successfully preserves the minority expert routing (94.6\% usage for 7B, 5.4\% for smaller models) and achieves a strictly lower overall deferral loss of $0.885 \pm 0.032$.

\textbf{Standard Cost-Sensitive Methods and Temperature Scaling.} We also emphasize that existing methods for imbalanced cost-sensitive multiclass classification are not directly applicable here. These methods generally aim to correct for imbalances in the \emph{label} distribution. They do not account for the instance-dependent nature of deferral decisions, where expert suitability varies across the input space regardless of the label. For example, as shown in Figure~\ref{fig:4cl3ex}, expert accuracy may depend entirely on the input region even when all labels are equally represented. 

The critical distinction between \MILD\ and standard temperature scaling or margin methods (such as \citep{cao2019learning}) is that standard methods adjust logits based on global, static class label frequencies for standard $0$-$1$ classification. In our two-stage deferral setting, \MILD\ applies these adjustments to experts while simultaneously incorporating dynamic, instance-dependent expert costs $\overline{c}_k(x,y)$. Deriving an $\sF$-consistent margin bound for this setting required substantial theoretical extensions beyond standard frequency-based adjustments. Consequently, standard cost-sensitive baselines (and static frequency-based margins) fail to capture the structure of the deferral problem, whereas the \TDEF\ baseline and \MILD\ are specifically tailored to this setting.

\section{Conclusion}
\label{sec:conclusion}
We presented principled algorithms to address the inherent imbalance
commonly encountered in deferral tasks. By reformulating the
minimization of the deferral loss as a cost-sensitive multi-class
classification problem, we derived algorithms that are both
theoretically sound and effective for handling expert imbalance.  Our
empirical results demonstrate the benefits of these new deferral
algorithms in imbalanced settings.

\ignore{
Additionally, our analysis has led to the design of general
cost-sensitive multi-class classification methods. Even in balanced
scenarios, these methods have the potential to enhance standard
cost-sensitive learning and structured prediction tasks 
(see Appendix~\ref{app:csl}).
}

\newpage
\section*{Impact Statement}

This paper presents work whose goal is to advance the field of Machine
Learning. There are many potential societal consequences of our work, none
of which we feel must be specifically highlighted here.

\bibliography{defid,add}
\bibliographystyle{icml2026}

\newpage
\appendix
\onecolumn

\renewcommand{\contentsname}{Contents of Appendix}
\tableofcontents
\addtocontents{toc}{\protect\setcounter{tocdepth}{3}} 
\clearpage

\section{Related Work}
\label{app:related-work}

\textbf{Learning to defer.} The \emph{single-stage learning to defer} paradigm has been
extensively studied, beginning with foundational research on learning
with abstention by \citet{CortesDeSalvoMohri2016,
  CortesDeSalvoMohri2016bis, CortesDeSalvoMohri2024}, and followed by extensive work on abstention and deferral \citep{madras2018learning,
  raghu2019algorithmic, mozannar2020consistent, wilder2021learning,
  pradier2021preferential, keswani2021towards, raman2021improving,
  liu2022incorporating, verma2022calibrated, charusaie2022sample,
  caogeneralizing, verma2023learning, MaoMohriZhong2024deferral,
  MaoMohriZhong2024predictor,MaoMohriZhong2024score,maorealizable,
  pmlr-v206-mozannar23a}. In this single-stage approach, a predictor
and a deferral function are learned jointly, with the deferral
function determining the best expert for each input. Specifically, in the abstention setting, where the cost function is constant, \citet{caogeneralizing, MaoMohriZhong2024score, MaoMohriZhong2024predictor} proposed surrogate losses that are Bayes-consistent \citep{Zhang2003,bartlett2006convexity, steinwart2007compare}. More generally, in the deferral setting, where the cost function depends on both the instance and the label, \citet{mozannar2020consistent, charusaie2022sample, verma2022calibrated, pmlr-v206-mozannar23a, maorealizable} proposed Bayes-consistent surrogate losses. Furthermore, \citet{verma2023learning, MaoMohriZhong2024deferral} extended the surrogate losses in \citep{verma2022calibrated, mozannar2020consistent} to the multi-expert deferral setting. In the specific case of abstention, the literature on selective classification provides methods for optimizing the generalization error of non-abstained samples under a fixed selection rate (see, for example, \citep{el2010foundations,wiener2011agnostic,el2012active,wiener2012pointwise,wiener2015agnostic,geifman2017selective,geifman2019selectivenet}).  However, these methods are not directly applicable to the deferral problem due to the presence of label-dependent costs and multiple experts. Extensions to the regression setting have also been studied in \citep{wiener2012pointwise,geifman2019selectivenet,jiang2020risk,zaoui2020regression,de2020regression,shah2022selective,li2023no,cheng2023regression} for abstention and in \citep{mao2024regression} for multi-expert deferral.  Additional research has expanded this framework to include deferring to populations \citep{tailor2024learning}, adversarial robustness \citep{montreuil2026adversarial}, top-$k$ learning \citep{montreuil2026why}, tailored query mechanisms \citep{montreuil2026optimal}, and other extensions \citep{montreuil2026beyond,montreuil2026online,montreuil2026learning,montreuil2026time}.

However, in many practical scenarios, strong predictors, such as a family of LLMs, are already available, and retraining them
alongside a deferral function can be computationally prohibitive.
Thus, the single-stage learning to defer framework and its associated
methods often overlook the practical constraints encountered in
real-world applications.
To address these limitations, \citet{mao2023two} introduced and
studied the \emph{two-stage learning to defer} framework, where the family of predictors is fixed and only the deferral function is learned. They provided non-asymptotic learning guarantees and effective algorithms, demonstrating strong empirical performance.
Specifically, they developed a novel family of surrogate loss
functions and algorithms with broad potential application, especially
in LLMs and other practical settings.  They proved that these
surrogate losses satisfy $\sH$-consistency bounds
\citep{awasthi2022h,awasthi2022multi,mao2023cross,mao2023structured,MaoMohriZhong2023ranking,MaoMohriZhong2023characterization,MaoMohriZhong2023rankingabs,mao2024h,mao2024multi,mao2025enhanced,MaoMohriZhong2025mastering,MaoMohriZhong2025principled,mao2025theory,zhong2025fundamental,desalvo2025budgeted,CortesMohriZhong2026mod,MohriZhong2026rllm,MohriZhong2026slin,mohri2025beyond,mohri2026principled,mohri2026generalized}, which are non-asymptotic,
hypothesis-set-specific upper bounds on the target estimation loss
expressed in terms of the surrogate estimation loss.  These bounds
provide stronger and more informative guarantees than
Bayes-consistency
\citep{Zhang2003,bartlett2006convexity,steinwart2007compare,mozannar2020consistent}, which only guarantees that minimizing the
surrogate loss over all measurable functions asymptotically minimizes
the target loss.  This approach differs from post-hoc methods \citep{okati2021differentiable,narasimhanpost} in that it can be used with existing predictors trained in the standard classification setting.

\textbf{Learning from imbalanced data.} A common strategy for addressing data imbalance involves oversampling
underrepresented classes or undersampling dominant ones
\citep{chawla2002smote,WallaceSmallBrodleyTrikalinos2011,
  KubatMatwin1997,QiaoLiu2008,han2005borderline,
  estabrooks2004multiple,
  liu2008exploratory,zhang2021learning}. Another related approach assigns
different loss penalties to different classes \citep{Iranmehr:2019,Masnadi-Shirazi:2010,elkan2001foundations,zhou2005training, zhao2018adaptive,zhang2018online, zhang2019online,sun2007cost,Fan:2017,cui2019class,jamal2020rethinking,GabidollaZharmagambetovCarreiraPerpinan2024}. However, these methods
lack strong theoretical justification, as they modify the training
distribution in ways that diverge from the true target
distribution. Empirically, their effectiveness is inconsistent and
often depends on extensive hyperparameter tuning \citep{VanHulse:2007}. In the deferral
setting, such techniques are even more problematic since they would
require assigning additional costs to experts, while the deferral
problem already incorporates instance-specific expert costs.

Beyond these data modification and cost-sensitive learning approaches, there have been a variety of techniques in the imbalanced multi-class classification setting, including logistic loss modifications \citep{lin2017focal,cao2019learning,tan2020equalization,jiawei2020balanced,hong2020disentangling,tian2020posterior,menon2020long,khan2019striking,menon2020long,ye2020identifying,kini2021label,zhu2023generalized,weilearning,LiXuBaoYangCongCaoHuang2024,cortes2025improved}, data augmentation \citep{wang2021rsg,zhugenerative,liuelta,gao2024enhancing}, representation learning \citep{liu2019large,cui2021parametric,gao2024distribution,meng2024learning,han2023wrapped}, decoupled training \citep{kang2019decoupling,zhong2021improving},  classifier
design \citep{tang2020long,yang2022inducing,kasarla2022maximum,shi2024long,dusimpro}, ensemble learning \citep{zhou2020bbn,xiang2020learning,wang2020long,cui2021reslt,zhang2021test,yangharnessing}, etc. We refer the reader to a recent survey by \citet{ZhangKangHooiYanFeng2023} for a more extensive list and details.

It remains an open question how these techniques can be extended and applied to the deferral problem with expert imbalance. A further challenge unique to deferral is that the learning distribution is defined over input-label pairs, whereas the imbalance we aim to correct concerns the distribution of experts, not labels. This distinction complicates the direct application of traditional imbalance-handling techniques. Furthermore, the traditional techniques mentioned above are often limited to empirical study and lack theoretical guarantees. Instead, this work designs a principled algorithm for deferral that effectively accounts for expert imbalance while preserving theoretical soundness and guarantees.

\section{Imbalance in Two-Stage Learning to Defer}
\label{app:imbalance}

\textbf{Notion of imbalance.} To define the notion of imbalance in the two-stage learning to defer setting, we refer readers to the definition of the deferral loss function $\ldef$
(Section~\ref{sec:deferral-loss}). In this context, $x$ denotes the input instance, and \emph{imbalance} refers to a situation where, for a large majority of instances, the same expert consistently incurs the lowest cost—that is, is considered the most suitable expert according to the cost function used in minimizing the deferral loss. This results in a highly skewed deferral pattern that over-relies on a small subset of experts, potentially not effectively using others who may be better suited for specific parts of the input space.

The term \emph{distribution of experts} refers to the empirical distribution over expert selections induced by the deferral policy across the input space. For example, in an LLM-based deferral system, imbalance may occur when one pretrained LLM performs significantly better than others across most inputs. If the deferral mechanism does not account for this, it may default to that model universally, ignoring specialized models that perform better on certain subpopulations.

Our experimental setup is designed to reflect precisely this kind of imbalance scenario—where a small subset of experts dominates unless appropriately regulated.

\textbf{Failure of standard deferral algorithms.} The observation that models trained on imbalanced datasets tend to underperform on minority categories is well-established in the imbalanced learning literature \citep{ZhangKangHooiYanFeng2023}. This phenomenon is often associated with \emph{long-tailed} distributions, where a few dominant classes receive the majority of the training data, leading classifiers to perform poorly on underrepresented classes—sometimes only marginally outperforming na\"ive baselines that always predict the majority class.

In our context, this translates to a small subset of experts being favored across most of the input space, resulting in a similarly long-tailed distribution over expert selections. Consequently, deferral algorithms trained under such imbalance tend to overfit to the dominant experts and not effectively use others—mirroring patterns observed in general imbalanced learning.

While this issue has not been extensively studied in the specific context of two-stage learning to defer, our work builds on this well-known phenomenon to motivate the need for more balanced expert use. In particular, we present severe expert imbalance scenarios in Appendix~\ref{app:exp-severe}, where baseline methods fail, and show that our proposed algorithm, \MILD, performs effectively to illustrate this point.

\section{Additional Experiments and Details}
\label{app:experiments}

\subsection{Image Classification Setup Details}
\label{app:exp-image}

To systematically evaluate robustness to expert imbalance, we constructed three expert setups with varying degrees of coverage and overlap. These setups simulate scenarios where experts have specialized sub-domains of competence. In particular, experts are designed to exhibit varying performance across different labels, a specific form of imbalance. We adopt this setup because label-based separation provides a clear and interpretable way to model expert specialization, and it can also serve as a proxy for more general forms of input partitioning. Our methods remain effective in scenarios where expert performance varies according to other characteristics of the input space (e.g., image features rather than labels).

\textbf{Setup I: ``7 vs.\ 2 vs.\ 1''.} Three synthetic experts were
      available. Expert 1 was always correct for the first 70\% of
      classes (e.g., classes 0–6 for CIFAR-10) and predicted uniformly
      at random for the remaining classes. Expert 2 was always correct
      for the next 20\% of classes (e.g., classes 7–8 for CIFAR-10)
      and predicted randomly for the remaining classes. Expert 3 was
      always correct for the last 10\% of classes (e.g., class 9 for
      CIFAR-10) and predicted randomly for all other classes.
      
\textbf{Setup II: ``5 vs.\ 2 vs.\ 2 vs.\ 1''.} 
Four experts were
      available. Expert 1 was always correct for the first 50\% of
      classes (e.g., classes 0–4 for CIFAR-10) and predicted uniformly
      at random for the remaining classes. Expert 2 was always correct
      for the next 20\% of classes (e.g., classes 5–6 for CIFAR-10)
      and generated random predictions for the remaining
      classes. Expert 3 was always correct for the next 20\% of
      classes (e.g., classes 7–8 for CIFAR-10) and generated random
      predictions for the remaining classes. Expert 4 was always
      correct for the last 10\% of classes (e.g., class 9 for
      CIFAR-10) and predicted randomly for all other classes.
      
\textbf{Setup III: ``4 vs.\ 2 vs.\ 2 vs.\ 1 vs.\ 1''.}   Expert 1 was always correct for the first 40\%
      of classes and predicted randomly for the remaining
      classes. Expert 2 was always correct for the next 20\% of
      classes and predicted randomly for the remaining classes. Expert
      3 was always correct for the next 20\% of classes and predicted
      randomly for the remaining classes. Expert 4 was always correct
      for the next 10\% of classes and predicted randomly for all
      other classes. Expert 5 was always correct for the last 10\% of
      classes and predicted randomly for all other classes.

We carried out experiments with two types of experts: synthetic and
real. The synthetic experts were generated precisely as described in the
three steps. Real experts, on the other hand, were trained on the
training data from their corresponding classes plus a 1\% fraction of
training data from the other classes. For example, in the ``7 vs.\ 2
vs.\ 1" setup on CIFAR-10, Expert 1 was trained using all training
data from classes 0–6 and 1\% of training data from classes 7–9.

We considered two types of cost functions. For the first type, we
chose the misclassification errors of the experts as the cost
functions: $c_j(x, y) = 1_{\expertexpert_{j}(x) \neq y}$. For the
second type, we chose the misclassification errors of the experts plus
the percentage of the domain where the expert is accurate as the cost
function. For example, in setup I (``7 vs.\ 2 vs.\ 1"), the cost
functions are chosen as $c_1(x, y) = 1_{\expertexpert_1(x) \neq y} +
0.7$, $c_2(x, y) = 1_{\expertexpert_2(x) \neq y} + 0.2$ and $c_3(x, y)
= 1_{\expertexpert_3(x) \neq y} + 0.1$ for Expert 1, Expert 2, and
Expert 3, respectively.

\subsection{LLM Routing on MMLU: Extended Details}
We provide further details on the MMLU experimental setup used in Section~\ref{sec:experiments}.
\label{app:exp-LLM}

\textbf{Models.} We adopted the \texttt{Qwen 2.5-Instruct} family of models \citep{qwen2.5} as experts due to their state-of-the-art performance in the open-weights category and their consistent architecture across sizes.
\begin{itemize}
    \item \textbf{Expert 1 (Strong):} \texttt{Qwen2.5-7B-Instruct}. Used as the generalist anchor.
    \item \textbf{Expert 2 (Medium):} \texttt{Qwen2.5-1.5B-Instruct}. Represents a balanced edge-device model.
    \item \textbf{Expert 3 (Tiny):} \texttt{Qwen2.5-0.5B-Instruct}. Represents an extremely lightweight speculative model.
\end{itemize}

\textbf{Dataset.} We constructed the routing dataset using the \textbf{MMLU} (Massive Multitask Language Understanding) benchmark \citep{hendrycks2020mmlu}. Specifically, we aggregated the `high\_school\_mathematics' and `high\_school\_world\_history' subsets to create a mix of reasoning-heavy and knowledge-heavy queries. The test set consisted of 2,000 samples.

\textbf{Cost Settings.}
We defined two distinct regimes to test the router's adaptability:
\begin{enumerate}
    \item \textbf{Setting (a) Error Only:} The cost is binary, $c_k(x, y) = 1_{\expertexpert_k(x) \neq y}$. This tests the router's ability to maximize accuracy without regard for compute. Since the 7B model is generally superior, this setting creates a massive imbalance where one expert is optimal $>80\%$ of the time.
    \item \textbf{Setting (b) Error + Cost:} The cost includes a normalized inference penalty, $c_k(x, y) = 1_{\expertexpert_k(x) \neq y} + \beta_k$. We set $\beta = [1.0, 0.6, 0.1]$. This creates a complex trade-off: the 0.5B model is ``optimal'' (lowest loss) for any query it gets right, while the 7B model is only optimal for queries where both smaller models fail.
\end{enumerate}

\subsection{Severe Expert Imbalance Analysis}
\label{app:exp-severe}

To further stress-test the algorithms, we examine a scenario with severe expert imbalance,
where the baseline method \TDEF\ fails, while our proposed algorithm,
\MILD, performs effectively.

\begin{table}[t]
\vskip 0.3in
    \caption{Comparison of our \MILD\ algorithm with
    \TDEF\ on CIFAR-10, CIFAR-100, SVHN under severe expert imbalance.}
    \centering
    \begin{tabular}{@{\hspace{0cm}}lllcc@{\hspace{0cm}}}
    \toprule
    \multirow{3}{*}{Method} & \multirow{3}{*}{Dataset} & Deferral Loss
    & \multicolumn{2}{c}{Ratio of Expert Deferral (\%)} \\
    \cmidrule{4-5}
    & & DL $\sim$ error rate & 1 & 2\\
    \midrule
    \TDEF & \multirow{2}{*}{CIFAR-10} & 0.0810 $\pm$ 0.0025 & 97.94 & 2.06\\
    \MILD & & \textbf{0.0125 $\pm$ 0.0016} & 90.52 & 9.48\\
    \cmidrule{1-5}
    \TDEF &  \multirow{2}{*}{CIFAR-100} & 0.0881 $\pm$ 0.0034 & 98.41 & 1.59\\
    \MILD & & \textbf{0.0192 $\pm$ 0.0037} & 91.56 & 8.44\\
    \cmidrule{1-5}
    \TDEF & \multirow{2}{*}{SVHN} & 0.0782 $\pm$ 0.0016 & 97.55 & 2.45\\
    \MILD & & \textbf{0.0086 $\pm$ 0.0012} & 90.05 & 9.95\\
    \bottomrule
    \end{tabular}
    \vskip -0.3in
    \label{tab:comparison-severe}
\end{table}

\textbf{Severe expert imbalance scenario:} For this scenario, we consider two experts, defining their cost functions based on misclassification errors:
$c_j(x, y) = 1_{\expertexpert_{j}(x) \neq y}$. This expert configuration, which reflects the real experts setting discussed in Section~\ref{sec:experiments}, involved training Expert 1 with all training data from 90\% of classes and 1\% from the remaining classes. Expert 2 was similarly trained using all data from 10\% of classes and 1\% from the remaining classes.

Table~\ref{tab:comparison-severe} shows that a naive application of the \TDEF\ algorithm in this scenario predominantly selects Expert 1, thereby failing to leverage Expert 2 despite its relevance to a nontrivial portion of the data. In contrast, \MILD\ consistently achieves significantly smaller deferral losses by selecting the experts much closer to their optimal allocation.

\subsection{Choice of \texorpdfstring{$\brho$}{rho}} 
\label{app:exp-rho}

As discussed in Section~\ref{sec:algorithms} (B. General cost-sensitive algorithm), while our general algorithm allows $\rho_j$ to be tuned freely over a range of values, the search is in fact guided by the theoretically optimal values derived from our analysis. Specifically, the theoretical guidance suggests choosing $\rho_j$ close to $\frac{(m_j \sfX_{j}^2)^{\frac{1}{3}}}{\ov \sfX} \ov \rho$. In the experiments, we use this expression to initialize the search range, and perform validation-based tuning in a small neighborhood around these theoretically motivated values, using a step size of $1$ over the interval $\bracket*{\frac{(m_j \sfX_{j}^2)^{\frac{1}{3}}}{\ov
  \sfX} - 5, \frac{(m_j \sfX_{j}^2)^{\frac{1}{3}}}{\ov
  \sfX} +5}$. Empirically, we observe that performance is quite robust within this neighborhood, which suggests that our theoretical estimation serves as a reliable prior for guiding the selection of $\rho_j$.

\section{Discussion on Alternative Baselines}
\label{app:confidence}

We briefly address the applicability of confidence-based methods to the two-stage multiple-expert deferral setting. (For a discussion on standard cost-sensitive methods, see Section~\ref{sec:further_analysis}.)

\textbf{Confidence-based methods.} Standard baselines that rely on thresholding prediction confidence are primarily designed for the single-expert (or abstention) setting, where the decision is binary and costs are typically constant. Extending this paradigm to the multi-expert setting is non-trivial, as it is unclear how to define optimal thresholds when experts possess heterogeneous costs and predictive capabilities. To the best of our knowledge, the algorithm proposed by \citet{mao2023two} is the only existing baseline specifically formulated for two-stage deferral with multiple experts.

While one might consider a \emph{cascading} approach, where experts are queried sequentially based on confidence, this strategy is inherently sensitive to the ordering of experts and implicitly assumes a fixed ranking of predictive strength. Furthermore, it ignores the \emph{instance-dependent} inference costs associated with each expert. In contrast, our method operates in a \emph{routing} framework, selecting the most suitable expert from a parallel ensemble based on both accuracy and cost, regardless of expert ordering. Moreover, even in the simpler abstention setting, confidence-based methods have been shown to be suboptimal \citep{CortesDeSalvoMohri2016} and are outperformed by learning-to-defer approaches \citep{mao2023two}. Crucially, neither approach accounts for expert imbalance, the central challenge addressed by our work.

\section{Reformulation of the Deferral Loss: Proofs of Lemma~\ref{lemma:deferral-loss-1} and Lemma~\ref{lemma:deferral-loss-2}}
\label{app:deferral-loss}

\DeferralLossOne*

\begin{proof}
For any $f \in \sF$ and $(x, y) \in \sX \times \sY$, we have
\begin{equation}
\begin{aligned}
\sfL_{\rm def}(f, x, y)
& =
\sum_{k = 1}^p c_k(x, y) 1_{\ff(x) = k}\\
& = \sum_{k = 1}^p c_k(x, y) \paren*{\sum_{k' = 1}^{p} 1_{\ff(x) \neq k'}1_{k' \neq k} - (p - 2)}\\
& = \sum_{k = 1}^p \sum_{k' = 1}^{p} c_k(x, y) 1_{\ff(x) \neq k'}1_{k' \neq k}
- (p - 2) \sum_{k = 1}^p c_k(x, y)\\
& = \sum_{k = 1}^{p} \paren*{\sum_{k' = 1}^p c_{k'}(x, y)1_{k' \neq k}} 1_{\ff(x) \neq k}
- (p - 2) \sum_{k = 1}^p c_k(x, y)\\
& = \sum_{k = 1}^{p} \paren*{\sum_{k' = 1}^p c_{k'}(x, y)1_{k' \neq k}} 1_{\rho_f(x, k) \leq 0}
- (p - 2) \sum_{k = 1}^p c_k(x, y),
\end{aligned}
\end{equation}
which completes the proof.
\end{proof}

\DeferralLossTwo*
\begin{proof}
For any $f \in \sF$ and $(x, y) \in \sX \times \sY$, we have
\begin{equation}
\begin{aligned}
\sfL_{\rm def}(f, x, y)
& =
\sum_{k = 1}^p c_k(x, y) 1_{\ff(x) = k}\\
& = \sum_{k = 1}^p c_k(x, y) \paren*{ 1 - 1_{\ff(x) \neq k}}\\
& = \sum_{k = 1}^p c_k(x, y) - \sum_{k = 1}^p \paren*{c_k(x, y) - 1} 1_{\ff(x) \neq k}
- \sum_{k = 1}^p 1_{\ff(x) \neq k}\\
& = \sum_{k = 1}^p \paren*{1 - c_k(x, y)} 1_{\ff(x) \neq k} + \sum_{k = 1}^p c_k(x, y) - (p - 1)\\
& = \sum_{k = 1}^p \paren*{1 - c_k(x, y)} 1_{\rho_f(x, k) \leq 0} + \sum_{k = 1}^p c_k(x, y) - (p - 1),
\end{aligned}
\end{equation}
which completes the proof.
\end{proof}

\newpage
\section{Margin Bound: Proof of Theorem~\ref{thm:margin-bound}}
\label{app:margin-bound}

\MarginBound*
\begin{proof}
Consider the family of functions taking values in $[0, 1]$:
\begin{equation*}
\sF' = \curl*{z = (x, k) \mapsto \sfL_{\brho}(f, x, k) \colon f \in \sF}.
\end{equation*}
By \citep[Theorem~3.3]{MohriRostamizadehTalwalkar2018}, with probability at least $1 - \delta$, for all $g \in \sF'$,
\begin{equation*}
\E[g(z)] \leq \frac{1}{m} \sum_{i = 1}^m g(z_i) + 2 \h \Rad_{S}(\sF') + 3 \sqrt{\frac{\log \frac{2}{\delta}}{2m}},
\end{equation*}
and thus, for all $f \in \sF$,
\begin{equation*}
\E[\sfL_{\brho}(f, x, k)] \leq \h \sE_{S, \brho}(f) + 2 \h \Rad_{S}(\sF') + 3 \sqrt{\frac{\log \frac{2}{\delta}}{2m}}.
\end{equation*}
Since $\sE_{\sfL}(f) \leq \sE_{\sfL_{\brho}}(f) = \E[\sfL_{\brho}(f, x, k)]$, we have
\begin{equation*}
\sE_{\sfL}(f) \leq \h \sE_{S, \brho}(f) + 2 \h \Rad_{S}(\sF') + 3 \sqrt{\frac{\log \frac{2}{\delta}}{2m}}.
\end{equation*}
Fix $f$, $(x_i, k_i)$ and $\rho > 0$, define $\Psi$ as follows:
\[\Psi([f(x_i, k)]_{k \in [\num]}) = \max_{k' \in [\num]} \curl*{c(x_i, k_i, k') \Phi_{\rho}\paren*{f(x_i, k_i) - f(x_i, k')}}.\] Then,
by the sub-additivity of the maximum operator, we can write
for any $f, \wt f \in \sF$:
\begin{align*}
  & \Psi([f(x_i, k)]_{k \in [\num]}) - \Psi([\wt f(x_i, k)]_{k \in [\num]})\\
  & \leq \max_{k' \in [\num]} \curl*{c(x_i, k_i, k') \Phi_{\rho}\paren*{f(x_i, k_i) - f(x_i, k')} - c(x_i, k_i, k') \Phi_{\rho}\paren*{\wt f(x_i, k_i) - \wt f(x_i, k')}}\\
  &  \leq \max_{k' \in [\num]} \curl*{\frac{2 c(x_i, k_i, k')}{\rho} \norm*{[f(x_i, k) - \wt f(x_i, k)]_{k \in [\num]}}_1}
  \tag{by $\frac{1}{\rho}$-Lipschitzness of $\Phi_{\rho}$}\\
  &  \leq \frac{2 \sqrt{p}}{\rho} \norm*{[f(x_i, k) - \wt f(x_i, k)]_{k \in [\num]}}_2.  
\end{align*}
Thus, $\Psi$ is $\frac{2 \sqrt{p}}{\rho}$-Lipschitz with respect to the $\norm*{\cdot}_{2}$ norm. Thus, by the vector contraction lemma \citep{Maurer2016,cortes2016structured}, $\h \Rad_S(\sF')$ can be bounded as follows:
\begin{align*}
\h \Rad_S(\sF')
\leq \frac{2 \sqrt{2 p}}{m}
\E_{\bepsilon}\bracket*{\sup_{f \in \sF} \curl*{\sum_{j = 1}^p \sum_{i \in S_j} \sum_{k = 1}^p  \e_{ik} \frac{f(x_i, k)}{\rho_j}}}
= 2 \sqrt{2 p} \, \h \Rad_{S, \brho}(\sF).
\end{align*}
This proves the second inequality. The first inequality, can be derived in the same way by using the first inequality of \citep[Theorem~3.3]{MohriRostamizadehTalwalkar2018}.
\end{proof}

\section{Derivation of Surrogate Losses}
\label{app:derivation}

Using the fact that $c(x, k, k') = 0$ for $k = k'$, we can rewrite $\sfL_{\brho}$ as
\[
\sfL_{\brho}(f, x, k) = \max_{k' \neq k} \curl*{c(x, k, k') \Phi_{\rho_k}(f(x, k) - f(x, k'))}.
\]
For $\Phi_\rho(t) \leq \Psi(t/\rho)$, we have:
\begin{align*}
  \sfL_{\brho}(f, x, k)
  & \leq \max_{k' \neq k} \curl*{c(x, k, k') \Psi\paren*{\frac{f(x, k) - f(x, k')}{\rho_k}}} .
\end{align*}
In particular, choosing $\Psi$ to be the logistic loss yields:
\begin{align*}
  \sfL_{\brho}(f, x, k)
  & \leq \max_{k' \neq k} \curl*{c(x, k, k') \log\bracket*{1 + \exp\paren*{\frac{f(x, k') - f(x, k)}{\rho_k}}}}.
\end{align*}
When $c(x, k, k')$ is independent of $(k, k')$, we can write:
\begin{align*}
  \sfL_{\brho}(f, x, k)
  & \leq \max_{k' \neq k} \curl*{C(x) \log\bracket*{1 + \exp\paren*{\frac{f(x, k') - f(x, k)}{\rho_k}}}}\\
  & = C(x) \max_{k' \neq k} \curl*{\log\bracket*{1 + \exp\paren*{\frac{f(x, k') - f(x, k)}{\rho_k}}}}\\
  & = C(x) \log\bracket*{1 + \max_{k' \neq k} \exp\paren*{\frac{f(x, k') - f(x, k)}{\rho_k}}}\\
  & \leq C(x) \log\bracket*{1 + \sum_{k' \neq k} \exp\paren*{\frac{f(x, k') - f(x, k)}{\rho_k}}}\\  
  & = C(x) \log\bracket*{\sum_{k' = 1}^p \exp\paren*{\frac{f(x, k') - f(x, k)}{\rho_k}}}.
\end{align*}
Otherwise, 
\begin{align*}
  \sfL_{\brho}(f, x, k)
  & \leq \max_{k' \neq k} \curl*{c(x, k, k') \log\bracket*{1 + \exp\paren*{\frac{f(x, k') - f(x, k)}{\rho_k}}}}\\
  & = \log \bracket*{\max_{k' \neq k} \bracket*{1 + \exp\paren*{\frac{f(x, k') - f(x, k)}{\rho_k}}}^{c(x, k, k')}}\\
  & \leq \log \bracket*{\sum_{k' \neq k} \bracket*{1 + \exp\paren*{\frac{f(x, k') - f(x, k)}{\rho_k}}}^{c(x, k, k')}}.
\end{align*}
Another choice is to set $\Psi$ as the comp-sum loss with $\tau \neq 1$, which yields:
\begin{align*}
  \sfL_{\brho}(f, x, k)
  & \leq \max_{k' \neq k} \curl*{c(x, k, k') \Phi^{\tau}\bracket*{\exp\paren*{\frac{f(x, k') - f(x, k)}{\rho_k}}}}\\
  & = \frac{1}{1 - \tau} \max_{k' \neq k} \curl*{c(x, k, k') \paren*{\bracket*{1 + \exp\paren*{\frac{f(x, k') - f(x, k)}{\rho_k}}}^{1 - \tau} - 1}}\\
  & \leq \frac{1}{1 - \tau} \sum_{k' \neq k} \curl*{c(x, k, k') \paren*{\bracket*{1 + \exp\paren*{\frac{f(x, k') - f(x, k)}{\rho_k}}}^{1 - \tau} - 1}}.
\end{align*}
When $c(x, k, k')$ is independent of $(k, k')$, we can write:
\begin{align*}
  \sfL_{\brho}(f, x, k)
  & \leq \frac{C(x)}{1 - \tau} \sum_{k' \neq k} \curl*{ \paren*{\bracket*{1 + \exp\paren*{\frac{f(x, k') - f(x, k)}{\rho_k}}}^{1 - \tau} - 1}}.
\end{align*}
\newpage
\section{\texorpdfstring{$\sF$}{H}-Consistency Bound:
  Proof of Theorem~\ref{thm:H-consistency-surrogate}}
\label{app:H-consistency}

\HConsistencySurrogate*
\begin{proof}
The
conditional error and the best-in-class conditional error of the
cost-sensitive zero-one loss can be expressed as follows:
\begin{align*}
  \E_{k} \bracket*{\sfL(f, x, k) \mid x} &= C(x) \sum_{k \in [\num]} \sfp(k | x) 1_{\rho_f(x, y) \leq 0}
  = C(x) \paren*{1 - \sfp(\ff(x) | x)},\\
\inf_{f \in \sF}\E_{k} \bracket*{\sfL(f, x, k) \mid x} 
& = C(x) \paren*{1 - \max_{k \in [\num]} \sfp(k | x)}.
\end{align*}
Let $k_{\max} = \argmax_{k \in [\num]} \sfp(k | x) $. Then, the difference between the two terms is given by:
\begin{align*}
  \E_{k} \bracket*{\sfL(f, x, k) \mid x} - \inf_{f \in \sF}\E_{k} \bracket*{\sfL(f, x, k) \mid x}
  = \sfp(k_{\max} | x) - \sfp(\ff(x) | x).
\end{align*}
 For the imbalanced cost-sensitive surrogate, the conditional error can be expressed as follows:
\begin{align*}
& \E_{k} \bracket*{\wt \sfL_{\brho}(f, x, k)\mid x}\\
&=  \sum_{k \in [\num]} \sfp(k | x) \Phi_{\rho_y}(\rho_f(x, y)) \\
& =  C(x) \sum_{k \in [\num]} \sfp(k | x) \curl*{\log\bracket*{ \sum_{k' = 1}^p \exp\paren*{\frac{f(x, k') - f(x, k)}{\rho_k}}}}\\
& = C(x) \sfp(k_{\max} | x) \log \paren*{\sum_{k'\in [\num]} e^{\frac{f(x, k') - f(x, k_{\max})}{\rho_{k_{\max}}}}} +  C(x)\sfp(\ff(x) | x) \log \paren*{\sum_{k'\in [\num]} e^{\frac{f(x, k') - f(x, \ff(x))}{\rho_{\ff(x)}}} }\\
& \qquad + C(x) \sum_{k \notin \curl*{k_{\max}, \ff(x)}} \sfp(k | x) \log \paren*{\sum_{k' \in [\num]} e^{\frac{f(x, k') - f(x, k)}{\rho_k}}}.
\end{align*}
For any $f \in \sF$ and $x\in \sX$, by the completeness
of $\sF$, we can always find a family of hypotheses $\curl*{\ov
  f_{\mu} \colon \mu \in \mathbb{R}} \subset \sF$ such that
$\ov f_{\mu}(x, \cdot)$ take the following values:
\begin{align}
\label{eq:value-h-mu}
\ov f_{\mu}(x, k) = 
\begin{cases}
  f(x, k) & \text{if $k \not \in \curl*{k_{\max}, \ff(x)}$}\\
  \log\paren*{\exp\bracket*{f(x, k_{\max})} + \mu} & \text{if $k = \ff(x)$}\\
  \log\paren*{\exp\bracket*{f(x, \ff(x))} - \mu} & \text{if $k = k_{\max}$}.
\end{cases} 
\end{align}
Note that the hypotheses $\ov f_{\mu}$ has the following property:
\begin{align}
\label{eq:property-h-mu}
\sum_{k \in [\num]}e^{\frac{f(x, k)}{\rho_{k'}}} = \sum_{k \in [\num]}e^{\frac{\ov f_{\mu}(x, k)}{\rho_{k'}}},\, \forall \mu \in \mathbb{R} \text{ and } k' \in [\num].
\end{align}
Thus, the best-in-class conditional error can be upper bounded as follows:
\begin{equation*}
\inf_{f \in \sF } \E_{k} \bracket*{\wt \sfL_{\brho}(f, x, k) \mid x} \leq \inf_{\mu \in \Rset} \E_{k} \bracket*{\wt \sfL_{\brho}(\ov f_{\mu}, x, k) \mid x}.
\end{equation*}
The conditional regret can be lower bounded as follows:
\begin{align*}
& \E_{k} \bracket*{\wt \sfL_{\brho}(f, x, k) \mid x} - \inf_{f \in \sF}\E_{k} \bracket*{\wt \sfL_{\brho}(f, x, k) \mid x}\\
&\geq \E_{k} \bracket*{\wt \sfL_{\brho}(f, x, k) \mid x} - \inf_{\mu \in \Rset} \E_{k} \bracket*{\wt \sfL_{\brho}(\ov f_{\mu}, x, k) \mid x}\\
& = \sup_{\mu\in \Rset} \bigg\{\sfp(k_{\max} | x)\paren*{ \log \paren*{\frac{\sum_{k' \in [\num]} e^{\frac{f(x, k')}{\rho_{k_{\max}}}}}{e^{\frac{f(x, k_{\max})}{\rho_{k_{\max}}}}}} - \log\paren*{\frac{\sum_{k' \in [\num]} e^{\frac{f(x, k')}{\rho_{k_{\max}}}}}{e^{\frac{f(x, \ff(x))-\mu}{\rho_{k_{\max}}}}}}}\\
& \qquad +  \sfp(\ff(x) | x)\paren*{ \log \paren*{\frac{\sum_{k' \in [\num]} e^{\frac{f(x, k')}{\rho_{\ff(x)}}}}{e^{\frac{f(x, \ff(x))}{\rho_{\ff(x)}}}}} - \log \paren*{\frac{\sum_{k' \in [\num]} e^{\frac{f(x, k')}{\rho_{\ff(x)}}}}{e^{\frac{f(x, k_{\max}) + \mu}{\rho_{\ff(x)}}}}}} \bigg\}\\
& \geq \sfp(k_{\max} | x)\log\bracket*{\frac{\paren*{e^{f(x,k_{\max})} + e^{f(x,\ff(x))}}\sfp(k_{\max} | x)}{e^{f(x,k_{\max})}\paren*{\sfp(k_{\max} | x)+\sfp(\ff(x) | x)}}}\\
&\qquad + \sfp(\ff(x) | x)\log\bracket*{\frac{\paren*{e^{f(x,k_{\max})}  + e^{f(x,\ff(x))}}\sfp(\ff(x) | x)}{ e^{f(x,\ff(x))}\paren*{\sfp(k_{\max} | x) + \sfp(\ff(x) | x)}}} 
\tag{Maximum is achieved by $\mu^* = \frac{e^{f(x, \ff(x))}\sfp(\ff(x) | x)\rho_{\ff(x)} - e^{f(x,  k_{\max})}\sfp(k_{\max} | x)\rho_{k_{\max}}}{\sfp(k_{\max}| x)\rho_{k_{\max}} + \sfp(\ff(x) | x)\rho_{\ff(x)}}$}\\
& \geq \sfp(k_{\max} | x)\log\bracket*{\frac{2 \sfp(k_{\max} | x)}{\paren*{\sfp(k_{\max} | x)+\sfp(\ff(x) | x)}}} + \sfp(\ff(x) | x)\log\bracket*{\frac{2 \sfp(\ff(x) | x)}{ \paren*{\sfp(k_{\max} | x) + \sfp(\ff(x) | x)}}} 
\tag{minimum is attained when $f(x, \ff(x)) = f(x, k_{\max})$}\\
& \geq \bracket*{\sfp(k_{\max} | x)+\sfp(\ff(x) | x)} \times \frac12\bracket*{ \abs*{\frac{\sfp(k_{\max} | x)}{\sfp(k_{\max} | x)+\sfp(\ff(x) | x)}-\frac12} +\abs*{\frac{\sfp(\ff(x) | x)}{\sfp(k_{\max} | x)+\sfp(\ff(x) | x)}-\frac12}}^2
\tag{Pinsker’s inequality \citep[Proposition~E.7]{MohriRostamizadehTalwalkar2018}}\\
& = \bracket*{\sfp(k_{\max} | x)+\sfp(\ff(x) | x)} \times \frac12 \bracket*{\frac{\sfp(k_{\max} | x)-\sfp(\ff(x) | x)}{\sfp(k_{\max} | x)+\sfp(\ff(x) | x)}}^2
\tag{$\sfp(k_{\max} | x)\geq \sfp(\ff(x) | x)$}\\
& \geq \frac12 \paren*{\sfp(k_{\max} | x) - \sfp(\ff(x) | x)}^2
\tag{$\sfp(k_{\max} | x)+\sfp(\ff(x) | x)\leq 1$}.
\end{align*}
By taking the expectation of both sides and using  Jensen's inequality, we obtain:
\begin{equation*}
\sE_{\sfL}(f) - \sE^*_{\sfL}(\sF) + \sM_{\sfL}(\sF) \leq \sqrt{2} \paren*{\sE_{\wt \sfL_{\brho}}(f) - \sE^*_{\wt \sfL_{\brho}}(\sF) + \sM_{\wt \sfL_{\brho}}(\sF)}^{\frac12},
\end{equation*}
which completes the proof.
\end{proof}

\section{Cost-Sensitive Learning and Structured Prediction}
\label{app:csl}

While our primary focus is addressing class imbalance in learning to
defer, our theoretical analysis and algorithm offer independent
contributions to cost-sensitive learning and related areas such as
structured prediction. We briefly discuss that in this section.

We introduced a margin-based upper bound on the cost-sensitive loss function:
\begin{align*}
  \sfL(f, x, k)
  & \leq \max_{k' \in [\num]} \curl[\Big]{c(x, k, k') \Phi_{\rho}\paren[\big]{f(x, k) - f(x, k')}}.
\end{align*}
This bound extends prior work by incorporating instance-dependent cost
functions ($x$-dependent), making it both more general and tighter than the upper
bound used in \citep*{cortes2016structured} (see Lemma 4 and the
surrogate losses on page 8), the closest related study on margin-based
cost-sensitive and structured prediction bounds.

For instance, adopting the hinge loss for $\Psi$ as an auxiliary
margin-based loss to upper bound $\sfL_{\brho}$ yields a loss
function that serves as a lower bound for the hinge-loss-type surrogate loss
in that publication (which coincides with the StructSVM loss
function), even in the balanced case ($\rho_k = \rho$ for all $k$).
Similarly, using the logistic loss for $\Psi$ as an auxiliary
margin-based loss function to upper bound $\sfL_{\brho}$ results in a
loss function that is upper bounded by the logistic-loss-type
surrogate in \citep*{cortes2016structured} (an extension of
Conditional Random Field (CRF) loss function), even in the balanced
case.

For the same reasons, our margin-based theoretical analysis yields
more favorable learning bounds, since our margin loss $\sfL_{\brho}$
serves as a lower bound for the multiplicative margin loss considered
in \citep*{cortes2016structured}, even in the balanced case. Thus,
this leads to improved learning bounds for structured prediction
compared to that study, as well as structured prediction algorithms
with stronger theoretical guarantees.

We leave a more detailed study of the theoretical and algorithmic
implications of our analysis for cost-sensitive learning and
structured prediction to future work.

\end{document}